\documentclass{article}

\usepackage[preprint]{corl_2025} %

\PassOptionsToPackage{dvipsnames,table}{xcolor}
\usepackage{graphicx}          %
\usepackage[table,dvipsnames]{xcolor}  %
\usepackage{url}               %
\usepackage{booktabs}          %
\usepackage{microtype}         %

\usepackage{amsmath, amssymb, amsfonts, amsthm}  %
\usepackage{nicefrac}          %
\usepackage{bm}                %
\usepackage{bbm}               %

\usepackage{tabularx}          %
\usepackage{multirow}          %
\usepackage{floatrow}          %
\usepackage{colortbl}          %

\usepackage[ruled, vlined, linesnumbered]{algorithm2e}  %
\DontPrintSemicolon  %

\usepackage{subcaption}        %
\usepackage{wrapfig}           %
\usepackage{adjustbox}         %

\usepackage[titletoc]{appendix}
\usepackage{titletoc}
\usepackage{etoc}              %
\usepackage{multicol}          %
\usepackage{chngpage}          %
\usepackage{soul}              %
\usepackage{csquotes}          %
\usepackage{yfonts}            %
\usepackage{etoolbox}          %
\usepackage{dirtytalk}
\usepackage{mdframed} %
\usepackage{listings}%
\definecolor{light-gray}{gray}{0.95}
\usepackage{thmtools, thm-restate}  %

\usepackage{tikz}              %

\definecolor{mydarkblue}{rgb}{0,0.08,0.45}
\definecolor{cardinalred}{rgb}{0.549,0.082,0.082}
\definecolor{digitalred}{rgb}{0.694,0.016,0.055}
\definecolor{digitalblue}{rgb}{0.0000, 0.4235, 0.7216}
\definecolor{paloalto}{rgb}{0.0902, 0.3686, 0.3294}
\definecolor{paloverde}{rgb}{0.1529, 0.6000, 0.5373}
\definecolor{paloverdelight}{rgb}{0.3490, 0.7020, 0.6627}
\definecolor{paloverdejy}{HTML}{1EAE94}
\definecolor{poppy}{rgb}{0.9137, 0.5137, 0.0000}
\definecolor{spirited}{rgb}{0.8784, 0.3098, 0.2235}
\definecolor{spiriteddark}{rgb}{0.7804, 0.2745, 0.1961}
\usepackage{hyperref}          %
\hypersetup{ %
    pdftitle={},
    pdfauthor={},
    pdfsubject={Proceedings of the \@conferenceordinal\/ Conference on Robot Learning (CoRL \@conferenceyear)},
    pdfkeywords={},
    pdfborder=0 0 0,
    pdfpagemode=UseNone,
    colorlinks=true,
    filecolor=mydarkblue,
    linkcolor=mydarkblue,
    citecolor=paloverdejy,
    urlcolor=digitalred,
    pdfview=FitH
}

\usepackage[numbers]{natbib}   %

\usepackage[capitalise,nameinlink]{cleveref}    
\crefformat{equation}{#2Eq.~#1#3}
\crefformat{figure}{#2Fig.~#1#3}
\crefformat{section}{#2\S#1#3}
\crefformat{subsection}{#2\S#1#3}
\crefformat{subsubsection}{#2\S#1#3}
\crefformat{assumption}{#2Assumption~#1#3}
\crefformat{assumption}{#2Assumption~#1#3}

\usepackage[subtle]{savetrees} %
\definecolor{RedOrange}{rgb}{0.9, 0.3, 0.0}

\newcommand{\calU}{\mathcal{U}}

\newcommand{\calN}{\mathcal{N}}

\newcommand{\calL}{\mathcal{L}}

\newcommand{\calH}{\mathcal{H}}

\newcommand{\calD}{\mathcal{D}}

\newcommand{\E}{\mathbb{E}}

\newcommand{\R}{\mathbb{R}}

\newcommand{\iid}{\overset{\textup{iid}}{\sim}}

\newcommand{\polinf}{\Psi_{\pi\text{-}\mathrm{inf}}}
\newcommand{\actinf}{\Psi_{a\text{-}\mathrm{inf}}}
\newcommand{\polinfest}{\widehat{\Psi}_{\pi\text{-}\mathrm{inf}}}
\DeclareMathOperator*{\argtop}{arg\text{ }top}

\newtheorem{definition}{Definition}

\newtheorem{proposition}{Proposition}
\newtheorem{task}{Task}

\newif\ifarxiv
\arxivtrue

\makeatletter

\renewcommand{\section}{%
  \@startsection{section}{1}{\z@}%
                {-2.0ex plus -0.5ex minus -0.2ex}%
                {1.0ex plus 0.2ex}
                {\large\bf\raggedright}%
}
\renewcommand{\subsection}{%
  \@startsection{subsection}{2}{\z@}%
                {-1.8ex plus -0.5ex minus -0.2ex}%
                {0.5ex plus 0.2ex}
                {\normalsize\bf\raggedright}%
}
\renewcommand{\subsubsection}{%
  \@startsection{subsubsection}{3}{\z@}%
                {-1.5ex plus -0.5ex minus -0.2ex}%
                {0.3ex plus 0.2ex}
                {\normalsize\bf\raggedright}%
}

\newcommand\footnoteref[1]{\protected@xdef\@thefnmark{\ref{#1}}\@footnotemark}
\makeatother

\newcommand{\basemethod}{\textsc{Cupid}}
\newcommand{\qualitymethod}{\textsc{Cupid-Quality}}

\title{CUPID: Curating Data your Robot Loves\\with Influence Functions\vspace{-6pt}}

\newcommand{\authorhref}[3][black]{\href{#2}{\textcolor{#1}{#3}}}
\author{
    \bfseries
    \authorhref{http://agiachris.github.io/}{Christopher Agia}\textsuperscript{1},\,\,
    \authorhref{https://rohansinha.nl/}{Rohan Sinha}\textsuperscript{1},\,\,
    \authorhref{https://yjy0625.github.io/}{Jingyun Yang}\textsuperscript{1},\\ \bfseries
    \authorhref{https://contactrika.github.io/}{Rika Antonova}\textsuperscript{2},\,\,
    \authorhref{https://profiles.stanford.edu/marco-pavone}{Marco Pavone}\textsuperscript{1,3},\,\,
    \authorhref{https://harukins.github.io/}{Haruki Nishimura}\textsuperscript{4},\,\,
    \authorhref{https://mashaitkina.weebly.com/}{Masha Itkina}\textsuperscript{4},\,\,
    \authorhref{https://web.stanford.edu/~bohg/}{Jeannette Bohg}\textsuperscript{1}\\\vspace{-8pt}\\
    \textsuperscript{1}\href{https://www.stanford.edu/}{Stanford University}, \textsuperscript{2}\href{https://www.cam.ac.uk/}{University of Cambridge},
    \textsuperscript{3}\href{https://www.nvidia.com/en-us/research/}{NVIDIA Research},
    \textsuperscript{4}\href{https://www.tri.global/}{Toyota Research Institute}
    \vspace{-18pt}
}

\begin{document}

\makeatletter
\def\blfootnote{\xdef\@thefnmark{}\@footnotetext}
\makeatother
\maketitle

\begin{figure}[h]
    \vspace{-6pt}
    \includegraphics[width=0.97\linewidth]{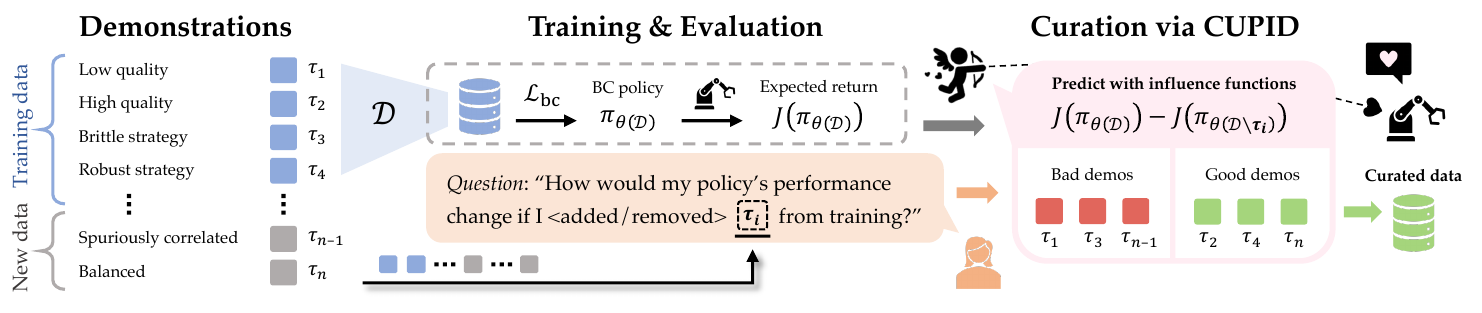}
    \vspace{-7pt}
    \caption{\small
        We present \textbf{CUPID}, a robot data curation method that leverages influence functions to predictively answer counterfactual questions about the effect of each demonstration on downstream policy performance.
    }
    \label{fig:teaser}
\end{figure}

\vspace{-8pt}
\begin{abstract}
    In robot imitation learning, policy performance is tightly coupled with the quality and composition of the demonstration data. 
Yet, developing a precise understanding of how individual demonstrations contribute to downstream outcomes---such as closed-loop task success or failure---remains a persistent challenge. 
We propose \basemethod{}, a robot data curation method based on a novel influence function-theoretic formulation for imitation learning policies.
Given a set of evaluation rollouts, \basemethod{} estimates the influence of each training demonstration on the policy's expected return. 
This enables ranking and selection of demonstrations according to their impact on the policy's closed-loop performance.
We use \basemethod{} to curate data by 1) filtering out training demonstrations that harm policy performance and 2) subselecting newly collected trajectories that will most improve the policy. 
Extensive simulated and hardware experiments show that our approach consistently identifies which data drives test-time performance. 
For example, training with less than 33\% of curated data can yield state-of-the-art diffusion policies on the simulated RoboMimic benchmark, with similar gains observed in hardware.
Furthermore, hardware experiments show that our method can identify robust strategies under distribution shift, isolate spurious correlations, and even enhance the post-training of generalist robot policies.
Videos and code are made available at: \href{https://cupid-curation.github.io}{https://cupid-curation.github.io}.
\blfootnote{Correspondence to: \href{mailto:cagia@cs.stanford.edu}{cagia@cs.stanford.edu}}
\vspace{-5pt}
\end{abstract}


\section{Introduction}\label{sec:introduction}
While some of the largest breakthroughs in deep learning have emerged from architectural innovations, data often remains an underrecognized yet critical driver of a model's overall performance.
In particular, the success of scaling vision and language models has been followed by a rising interest in data attribution~\cite{grosse2023studying, park2023trak, engstrom2024dsdm}---methods that causally link model behavior to training data---and in automatic data curation algorithms~\cite{lee2021deduplicating, tirumala2023d4, albalak2024survey}, grounded in the idea that not all data points contribute equally, or even positively, to a model's performance. As parts of the robotics community scale imitation learning and robotics datasets become increasingly diverse~\cite{o2024open, khazatsky2024droid}, developing a deeper understanding of (i) how demonstration data shapes policy behavior and (ii) how we can extract maximum utility from training datasets will be imperative to advancing policy performance toward reliable, open-world deployment.

Curating data for robot imitation learning has been the focus of several recent works~\cite{kuhar2023learning, hejna2025robotdatacurationmutual, chen2025curating}.
A common approach retains demonstrations deemed most valuable under a heuristic, \textit{task-agnostic quality} metric, resulting in a smaller dataset curated offline~\cite{hejna2025robotdatacurationmutual}. 
This approach typically rests on the implicit assumption that the designed quality metric aligns well with the policy’s downstream performance---an assumption that may not hold uniformly across diverse robotics tasks.
While recent efforts attempt to learn \textit{performance-correlated} heuristics using online policy experience~\cite{chen2025curating}, they do not establish strong causal links between training data and policy behavior. 
As a result, these methods risk misattributing the root cause of policy success or failure with respect to the training data~\cite{de2019causal}.

In this work, we formally define data curation in imitation learning as the problem of identifying which expert demonstrations maximally contribute to the policy’s expected return. 
We then introduce \basemethod{} (\textsc{Cu}rating \textsc{P}erformance-\textsc{I}nfluencing \textsc{D}emonstrations), a data curation method that directly targets this objective by leveraging influence functions~\cite{koh2017understanding, koh2019accuracy}---a technique popularized in the data attribution literature~\cite{hammoudeh2024training}---to identify which demonstrations influenced a policy’s predictions during closed-loop execution.
We show that a demonstration’s influence on expected return decomposes into a tractable sum over its state-action transitions and can be efficiently approximated using a \texttt{REINFORCE}-style estimator~\cite{sutton1999policy} given a set of policy rollouts. 
Ranking demonstrations by their estimated performance impact facilitates curation in two settings: (a) filtering existing demonstrations from training sets and (b) selecting high-impact demonstrations from newly collected or pre-collected data---whereas prior work focuses solely on filtering~\cite{hejna2025robotdatacurationmutual, chen2025curating}. 
Finally, while our approach offers a general and effective standalone signal for curating demonstration data, we investigate its combined use with task-agnostic quality metrics (also derived from influence scores), identifying conditions under which the integration of performance- and quality-based metrics strengthens or weakens overall curation performance.

Our contributions are three-fold: (1) We formulate robot data curation as the problem of valuating demonstrations in accordance with their downstream impact on policy performance; (2) We propose \basemethod{}, a novel approach for curating imitation learning datasets based on influence functions, causally linking demonstrations to the policy’s expected return; (3) We characterize the conditions under which the integration of task-agnostic quality metrics strengthens performance-based data curation, providing practical insights into when such integration is beneficial.
Extensive simulation and hardware experiments show that curation with \basemethod{} significantly improves policy performance in mixed-quality regimes, even when using only a fraction of the training data.
Moreover, it identifies robust strategies under test-time distribution shifts and can disentangle spurious correlations in training data that hinder generalization---all by observing policy outcomes alone, without requiring additional supervision.

\section{Related Work}\label{sec:related-work}
\textbf{Data Curation in Robotics.} 
Assembling larger and more diverse datasets has been central to scaling efforts in robot imitation learning~\cite{o2024open, khazatsky2024droid, rt12022arxiv, rt22023arxiv, octo_2023, pmlr-v270-kim25c, black2410pi0}, yet how to extract greater utility from these datasets remains an open question.
Several works have explored data augmentation~\cite{mandlekar2023mimicgen, yu2023scaling, mandi2022cacti, smith2024steer, pmlr-v270-zawalski25a} and mixture optimization~\cite{pmlr-v270-hejna25a}.
Only recently has attention shifted to valuating individual demonstrations for data curation~\cite{kuhar2023learning, hejna2025robotdatacurationmutual, chen2025curating}.
\citet{hejna2025robotdatacurationmutual} estimate demonstration quality offline via mutual information---without considering policy performance---while \citet{chen2025curating} train classifiers to distinguish successful and failed rollouts across policy checkpoints. In contrast, we directly measure the causal influence of each demonstration on the policy's expected return, providing a signal that (a) does not require observing both successes and failures, (b) uses only a single policy checkpoint, (c) is robust to spurious correlations in the policy's rollout distribution, and (d) naturally extends to selecting new data, whereas~\cite{hejna2025robotdatacurationmutual, chen2025curating} only filter existing data.
Concurrent to our work is DataMIL~\cite{dass2025datamil}, which uses datamodels to select from large multi-task datasets with an offline metric, whereas we focus on single-task curation with an influence measure that directly reflects closed-loop returns from online policy rollouts.

\textbf{Data Attribution outside Robotics.}
Data attribution methods model the relationship between training data and learned behavior, with applications in model interpretability~\cite{park2023trak, shah2023modeldiff}, data valuation~\cite{ghorbani2019data, choe2024your}, machine unlearning~\cite{georgiev2024attribute}, and more~\cite{madry2024icml}.
Recent work has focused on improving the accuracy of data attribution methods~\cite{basu2021influence, bae2022if, ilyas2025magic}, such as influence functions~\cite{koh2017understanding, koh2019accuracy}, and extending them to increasingly complex generative architectures~\cite{grosse2023studying, zheng2023intriguing, georgiev2023journey}.
A related line of research explores improving language model pre-training~\cite{engstrom2024dsdm} and fine-tuning~\cite{xia2024less, liu2024tsds, engstrom2025optimizing} through data selection.
However, these settings typically assume aligned training and evaluation objectives (i.e., prediction loss) and access to test-time labels.
In contrast, robot imitation learning involves an objective mismatch: policies are trained via supervised learning but evaluated through closed-loop environment interactions, where task success depends on many sequential predictions and ground-truth action labels are unavailable at test-time.

\section{Background: Data Attribution via Influence Functions}\label{sec:background}

At a high-level, \textbf{the goal of data attribution} methodologies is to explicitly relate model performance and behavior to the training data, so that we can answer \emph{counterfactual} questions about the contribution of training samples towards test-time predictions. 
Consider a standard supervised learning setting, where we fit model parameters $\theta$ on a given training dataset $\mathcal{D} := \{z^1, \ldots, z^n\}$ of input-label pairs $z^i=(x^i, y^i) \in \mathcal{Z}$ with $\theta(\mathcal{D}) =  \arg \min_{\theta'} \{\calL(\theta'; \calD) := \frac{1}{n} \sum_{i=1}^n \ell(z^i; \theta')\}$. Moreover, let $f(\hat{z}; \theta) \in \R$ be any chosen performance metric on a test sample $\hat{z} = (\hat{x}, \hat{y}) \in \mathcal{Z}$ given model parameters $\theta$ (e.g., cross-entropy loss for a classifier). Then, a data attribution method $\Psi^{\mathrm{out}}: \mathcal{Z} \times \mathcal{Z} \rightarrow \mathbb{R}$ aims to approximate the change in the performance metric $f$ if we were to exclude sample $z^i$ from the model's training data. That is, we aim to design $\Psi^{\mathrm{out}}$ such that $\Psi^{\mathrm{out}}(\hat{z}, z^i) \approx  f\left(\hat{z}; \theta(\mathcal{D} \setminus \{z^i\})\right) - f(\hat{z}; \theta(\mathcal{D}))$.

\textbf{The influence function} is a data attribution technique that approximates $\Psi^{\mathrm{out}}$ \emph{without} retraining any models~\cite{hammoudeh2024training}. 
Consider perturbing the training objective as $\calL_{\epsilon, z}(\theta'; \calD) := \calL(\theta'; \calD) + \epsilon \ell(z, \theta'),$ where we add an infinitesimal weight $\epsilon$ on the loss of some sample $z$ to $\calL$. The \emph{influence function} estimates the change in the performance metric $f$ as a function of $\epsilon$ with a first-order Taylor approximation as
\begin{align}
    \Psi_{\text{inf}}(\hat{z}, z)  &:= \frac{df(\hat{z}; \theta)}{d \epsilon}\bigg|_{\epsilon=0} = -\nabla_\theta f(\hat{z}; \theta(\mathcal{D}))^\top  H_{\theta}^{-1}  \nabla_\theta \ell(z; \theta(\mathcal{D})),
        \label{eq:influence-fn}
\end{align}
where $H_{\theta} = \frac{1}{n} \sum_{i=1}^{n} \nabla_\theta^2 \ell(z^i; \theta(\mathcal{D}))$ denotes the Hessian of the training loss\footnote{\label{fn:track} To reduce the computational cost of \cref{eq:influence-fn}, we use TRAK~\cite{park2023trak}, which leverages random projections and a Gauss-Newton Hessian approximation for efficient influence estimation. This also makes the influence function amenable to the non-smooth, non-convex loss functions in practical deep learning problems, so we assume \cref{eq:influence-fn} is well-defined throughout this paper.}~\cite{koh2017understanding}. Therefore, we can use the influence function to directly approximate the \emph{leave-one-out} influence $\Psi^{\mathrm{out}}$ of a sample $z^i\in\calD$ as $\Psi^{\mathrm{out}}_{\mathrm{inf}}(\hat{z}, z^i) := -\frac{1}{n}\Psi_{\mathrm{inf}}(\hat{z}, z^i)$. In addition, for $z \not \in \calD$ we similarly define the \emph{add-one-in} influence as $\Psi^{\mathrm{in}}_{\mathrm{inf}}(\hat{z}, z):= \frac{1}{n} \Psi_{\mathrm{inf}}(\hat{z}, z) \approx f(\hat{z}; \theta(\calD \cup \{z\})) - f(\hat{z}; \theta(\calD))$ with $z$ excluded from the Hessian $H_{\theta}$.

\section{Problem Formulation}\label{sec:formulation}
\textbf{Imitation Learning (IL):}
The objective of this work is to understand how demonstration data contributes to closed-loop performance in robot imitation learning. Thus, we consider a Markov Decision Process $\langle \mathcal{S}, \mathcal{A}, \mathcal{T}, R, \rho_0 \rangle$ with state space~$\mathcal{S}$, action space~$\mathcal{A}$, transition model~$\mathcal{T}$, reward model~$R$, initial state distribution~$\rho_0$, and finite horizon~$H$. We train a policy~$\pi_\theta$ to minimize a behavior cloning (BC) objective, i.e., $\theta = \arg\min_{\theta'} \{\mathcal{L}_{\text{bc}}(\theta'; \mathcal{D}) := \frac{1}{|\calD| H}\sum_{\xi^i\in\calD}\sum_{(s, a) \in \xi^i} \ell(s, a; \pi_{\theta'})\}$, using a dataset of $n$ expert demonstrations $\mathcal{D} = \{\xi^1, \ldots, \xi^n\}$. Each demonstration $\xi^i = ((s^i_0, a_0^i), \ldots, (s_H^i, a_H^i))$ consists of a state-action trajectory where the robot successfully completes the task. 
We treat a trajectory $\tau = (s_0, a_0, \ldots, s_{H})$ as either a \emph{success} or a \emph{failure}, corresponding to the binary returns $R(\tau) = 1$ and $R(\tau) = -1$ respectively.

Therefore, in IL, we train the policy $\pi_\theta$ to match the distribution of successful behaviors in $\calD$, rather than directly maximize its expected return $J(\pi_\theta) := \E_{p(\tau|\pi_\theta)}[R(\tau)]$.
As a result, the policy's performance is intimately linked to the relative suboptimality of the demonstration data---a function of its quality and composition---not just to validation losses, model capacity, or bias-variance tradeoffs. This makes it extremely challenging to systematically improve performance. 
Recent works underscore that simply scaling demonstration collection may result in datasets that contain substantial redundancies and behaviors that may actually harm policy performance, even though $R(\xi^i) = 1$ for all demonstrations $\xi^i \in \calD$~\cite{belkhale2023data}.

\textbf{Robot Data Curation:} 
While several recent works propose intuitive measures of quality to curate data, we find that such heuristics can misalign with how deep models actually learn, sometimes even worsening test-time performance compared to randomly choosing samples
(see \cref{sec:experiments}). Therefore, we first formally define robot data curation as the problem of identifying demonstration data that maximizes the policy's closed-loop performance.
In particular, assume that we have a \emph{base policy} $\pi_{\theta}$ trained on the demonstration data $\calD$. We consider two settings that are essential to a policy debugging toolchain. The first is that of \emph{data filtering}, where our goal is to identify and remove redundant or harmful demonstrations from $\calD$ that may be limiting the performance of the base policy $\pi_\theta$.

\begin{task}[Filter-$k$ demonstrations]\label{task:filter-k}
    Let $\Xi^-_k = \left\{ S \subseteq \mathcal{D} \;\middle|\; |S| = k \right\}$ denote all possible $k$-demonstration subsets of the training dataset $\mathcal{D} = \{\xi^1, \ldots, \xi^n\}$, where $k \leq n$. Determine which $k$ demonstrations should be removed from $\mathcal{D}$ to maximize policy performance with respect to the task objective $J$. That is, find 
    \begin{align*}
        S^\star =\;& \arg \max _{S \in \Xi^-_k} J(\pi_{\theta}) \quad \mathrm{s.t.} \quad \theta = \arg \min_{\theta'} \ \mathcal{L}_{\mathrm{bc}}(\theta'; \mathcal{D} \setminus S).
    \end{align*}
\end{task}
The second is that of \emph{data selection}, where we seek to guide the subselection of new demonstration data to maximally improve our base policy, given a fixed budget.
\begin{task}[Select-$k$ demonstrations]\label{task:select-k}
    Let $\Xi^+_k = \left\{ S \subseteq \mathcal{H} \;\middle|\; |S| = k \right\}$ denote all possible $k$-demonstration subsets of a holdout dataset $\mathcal{H} = \{\xi^1, \ldots, \xi^{n'}\}$, where $k \leq n'$. Determine which $k$ demonstrations should be added to $\mathcal{D}$ from $\mathcal{H}$ to maximize policy performance with respect to the task objective $J$. That is, find 
    \begin{align*}
        S^\star =\;& \arg \max_{S \in \Xi^+_k} J(\pi_{\theta}) \quad \mathrm{s.t.} \quad \theta = \arg \min_{\theta'} \ \mathcal{L}_{\mathrm{bc}}(\theta'; \mathcal{D} \cup S). 
    \end{align*} 
\end{task}
In \cref{task:select-k}, we consider the problem of identifying the most impactful trajectories from a newly collected batch of demonstrations or from an existing pre-collected dataset, akin to performing quality control.

\textbf{Policy Testing \& Evaluation:} To make progress on \cref{task:filter-k} and \cref{task:select-k}, we assume access to a small dataset of $m$ rollouts $\mathcal{D}_\tau = \{\tau^1, \ldots, \tau^m\} \iid p(\tau | \pi_\theta)$ of the base policy $\pi_\theta$ along with their associated returns $\{R(\tau^1), \ldots, R(\tau^m)\}$ to estimate $J(\pi_\theta)$. This aligns with how we currently evaluate policies in practice~\cite{vincent2024generalizable}, despite lacking principled strategies to leverage evaluations towards BC policy improvement.

\begin{figure}[t]
    \centering
    \includegraphics[width=\linewidth]{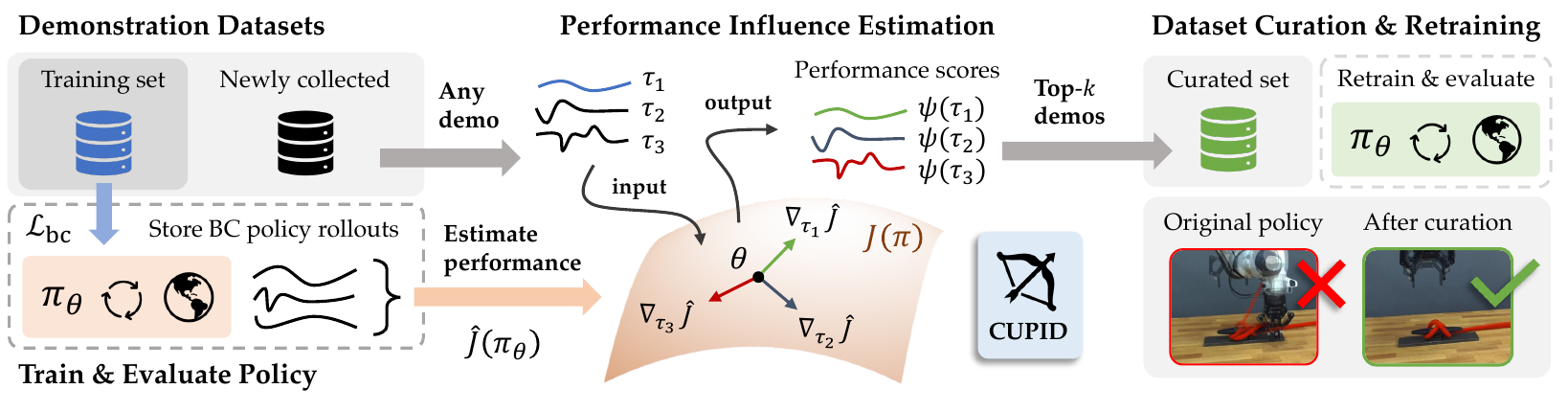}
    \vspace{-18pt}
    \caption{\small
        \textbf{Data curation with \basemethod{}.} Upon training a policy on a set of demonstrations using behavior cloning, we evaluate it online to collect closed-loop rollout trajectories and estimate the policy’s expected return. \basemethod{} ranks demonstration based on their measured influence on this performance estimate and selects the top-$k$. Thus, curating with \basemethod{} results in a dataset of demonstrations that most strongly influences closed-loop policy success.
    }
    \label{fig:method-overview}
\end{figure}

\section{CUPID: Curating Performance-Influencing Demonstrations}\label{sec:method}
In the literature, the desiderata for curating demonstration data appears diverse and often context specific, with recent works valuating demonstrations upon heuristic measures of similarity~\cite{chen2025curating}, compatibility~\cite{gandhi2023eliciting}, uncertainty~\cite{cui2019uncertainty}, and information gain~\cite{hejna2025robotdatacurationmutual}. \textbf{The key insight} of this work
is that solving curation problems, i.e., \cref{task:filter-k} and \cref{task:select-k} (\cref{sec:formulation}), requires causally connecting training data to the policy's closed-loop performance. 
Therefore, we first adapt techniques from data attribution, as defined in \cref{sec:background}, to directly compute the influence of a training demonstration on the performance of a policy. This allows us to use our \emph{performance influence} to directly curate data in alignment with our objectives.

\subsection{Demonstration-Performance Influence}\label{sec:performance-influence}
While existing data attribution methods can trace validation losses back to the training set $\calD$ for curation purposes, the BC loss is not always reflective of a policy's closed-loop performance~\cite{ross2011reduction}. Thus,
we must first develop an analogous notion of the influence function to capture the impact of a \emph{demonstration trajectory} on the \emph{closed-loop performance} of an imitation learning policy.
To do so, we group the BC training objective into trajectory-level losses by introducing  $\ell_{\mathrm{traj}}(\xi ; \pi_{\theta'}) := \frac{1}{H}\sum_{(s,a) \in \xi}\ell(s, a; \pi_{\theta'})$, so that $\calL_{\mathrm{bc}}(\theta'; \calD) = \frac{1}{|\calD|} \sum_{\xi^i \in \calD}\ell_{\mathrm{traj}}(\xi^i ; \pi_{\theta'})$. 
We now formally define the \emph{performance influence} of a demonstration as the application of the influence function (see \cref{eq:influence-fn}) on the policy's expected return:

\begin{definition}[Performance Influence]\label{def:polinf} Let $\xi$ be a demonstration of interest. Suppose we train a policy $\pi_\theta$ to minimize the perturbed BC objective $\calL_{\mathrm{bc}}^{\epsilon, \xi}(\theta' ; \calD) := \calL_{\mathrm{bc}}(\theta'; \calD) + \epsilon \ell_{\mathrm{traj}}(\xi; \pi_{\theta'})$. Then, demonstration $\xi$'s  \emph{\textbf{performance influence}} is the derivative of the policy's expected return $J(\pi_\theta)$ with respect to the weight $\epsilon$. That is,
\begin{align}\label{eq:polinf}
\polinf(\xi) &:= \frac{dJ(\pi_\theta)}{d\epsilon}\bigg|_{\epsilon=0} \nonumber = -\nabla_\theta J(\pi_\theta)^\top H_{\mathrm{bc}}^{-1} \nabla_\theta \ell_{\mathrm{traj}}(\xi; \pi_\theta),
\end{align}
where $H_{\mathrm{bc}} := \nabla^2_\theta\calL_{\mathrm{bc}}(\theta ; \calD)$ denotes the Hessian of the BC objective.
\end{definition}
In essence, \cref{def:polinf} enables us to predictively answer the counterfactual: ``How would the policy's expected return change if we upweighted---or by negating, downweighted---the loss on a demonstration $\xi$ during training?''
While \cref{def:polinf} neatly aligns with the standard definition of the influence function in \cref{eq:influence-fn}---using $J$ as the performance metric and $\ell_{\mathrm{traj}}$ as the demonstration-level loss---we distinguish the \emph{performance influence} from the standard influence function~\cite{koh2017understanding} for two key reasons: (1) The performance influence attributes the \emph{outcome} of a policy's sequential decisions to time-series demonstrations, whereas the existing techniques discussed in \cref{sec:background} only relate an individual labeled prediction to a single training sample; (2) We cannot directly compute $\polinf$ because the policy's expected return $J(\pi_\theta)$ depends on the unknown transition dynamics and reward function.
To alleviate these challenges, we show that we can decompose the \emph{performance influence} into influence scores of individual action predictions, which we define as the \textit{action influence}.
\begin{definition}[Action Influence]\label{def:actinf} The \emph{\textbf{action influence}} of a state-action pair $(s,a)$ on a test state-action pair $(s',a')$ is the influence of $(s,a)$ on the policy's log-likelihood $ \log \pi_\theta(a'|s')$. That is,
    \begin{equation}\label{eq:actinf}
    \actinf((s',a'), (s,a)) := -\nabla_\theta \log\pi_\theta(a'|s')^\top H_{\mathrm{bc}}^{-1}\nabla_\theta \ell(s, a; \pi_\theta).
\end{equation}
\end{definition}
The advantage of the \emph{action influence} is that we can easily compute the quantities in \cref{eq:actinf} given the policy weights $\theta$ and the training demonstrations $\calD$, e.g., using the attribution methods discussed in \cref{sec:background}.
However, we emphasize that computing \emph{action influences} over state-action samples from a policy rollout $\tau \sim p(\tau | \pi_\theta)$ only tells us what demonstration data led to the policy taking those actions, without ascribing value to the resulting outcome (e.g., success or failure).
We now show that the performance influence decomposes into the sum of individual action influences, weighted by the trajectory return $R(\tau)$.

\begin{proposition}\label{prop:polinf}
Assume that $\theta(\calD) = \arg\min_{\theta'} \calL_{\mathrm{bc}}(\theta'; \calD)$, that $\calL_{\mathrm{bc}}$ is twice differentiable in $\theta$, and that $H_{\mathrm{bc}} \succ 0$ is positive definite (i.e., $\theta(\calD)$ is not a saddle point)\footnoteref{fn:track}. Then, it holds that\footnote{Note that the fraction $1/H$ appears from the assumption that all trajectories have equal length, which we make purely for notational simplicity without loss of generality. We refer to \cref{appx:proof-length} for the variable length case.}
\begin{equation}\label{eq:policy-inf-deriv}
    \polinf(\xi) = \E_{\tau\sim p(\tau|\pi_\theta)}\bigg[\frac{R(\tau)}{H} \sum_{(s',a')\in\tau}\sum_{(s,a)\in\xi}\actinf\big((s',a'),(s,a)\big)\bigg].
\end{equation}
\end{proposition}
In brief, we prove \cref{prop:polinf} using the log-derivative trick underlying policy gradient methods~\cite{sutton1999policy, williams1992simple} to decompose $\polinf$ into $\actinf$ (see \cref{appx:proof} for proof).
Because \cref{prop:polinf} relates the performance influence to the average action influence that a demonstration $\xi$ has on the closed-loop distribution of policy rollouts, \cref{prop:polinf} directly provides a method to estimate $\polinf$: \newline
\textbf{Estimate $\polinf$:} First, evaluate the policy $\pi_\theta$ online to gather a set of rollouts $\mathcal{D}_\tau = \{\tau^1, \ldots, \tau^m\} \iid p(\tau | \pi_\theta)$ and their associated returns $\{R(\tau^1), \ldots, R(\tau^m)\}$. Then, construct an empirical estimate of the performance influence $\polinfest$ using \cref{eq:policy-inf-deriv}, by averaging action influences across the rollouts in $\calD_\tau$. 

We illustrate the performance influence in \cref{fig:method-overview} and summarize its estimation procedure in \cref{alg:polinf-est}.
\begin{algorithm}[H]
    \small
    \caption{Performance Influence}\label{alg:polinf-est}
    \textbf{Input:} Policy $\pi_\theta$, training data $\calD$, demonstration $\xi$, data attribution method $\Psi$ \\
    Collect rollouts $\mathcal{D}_\tau = \{\tau^1, \ldots, \tau^m\} \iid p(\tau | \pi_\theta)$ and returns $\{R(\tau^1), \ldots, R(\tau^m)\}$ \\
    Use $\Psi, \calD$ to compute $\actinf((s',a'), (s,a))$ for all $(s',a') \in \calD_\tau$, $(s,a) \in \xi$ \\
    Estimate $\widehat{\Psi}_{\pi\text{-}\mathrm{inf}}(\xi) := \frac{1}{m}\sum_{\tau^i \in \calD_\tau} \frac{R(\tau^i)}{H} \sum_{(s',a') \in \tau^i} \sum_{(s,a) \in \xi} \actinf((s',a'), (s,a))$. \\
    \textbf{Output:} Estimated performance influence $\widehat{\Psi}_{\pi\text{-}\mathrm{inf}}(\xi)$
\end{algorithm}

\subsection{Data Curation with Performance Influence}\label{sec:methods-curation}
In this section, we leverage the performance influence $\polinf,$ which we developed in \cref{sec:performance-influence}, to curate data towards the filtering and selection tasks (\cref{task:filter-k} and \cref{task:select-k}) defined in \cref{sec:formulation}. In particular, we use the estimates of $\polinf$ to make the following first-order Taylor approximations on the \emph{leave-one-out} and \emph{add-one-in} influence (as defined in \cref{sec:background}) of a demonstration trajectory as
\begin{minipage}{\linewidth}
\small
\begin{equation*}
    \polinf^{\mathrm{out}}(\xi) := -\frac{\polinfest(\xi)}{|\calD|} \approx J(\pi_{\theta(\calD \setminus \{\xi\})}) - J(\pi_{\theta(\calD)}), \quad
    \polinf^{\mathrm{in}}(\xi):= \frac{\polinfest(\xi)}{|\calD|} \approx  J(\pi_{\theta(\calD \cup \{\xi\})}) -J(\pi_{\theta(\calD)}).
\end{equation*}
\end{minipage}

Then, we use the \emph{leave-one-out} and \emph{add-one-in} influences to counterfactually estimate the change in expected return when removing or adding a set of demonstrations $S$ with a linear approximation as 
$\Delta \widehat{J}(\pi_{\theta(\calD \setminus S)}) \propto \frac{1}{|S|}\sum_{\xi\in S}\polinf^{\mathrm{out}}(\xi)$ and $\Delta \widehat{J}(\pi_{\theta(\calD \cup S)}) \propto \frac{1}{|S|}\sum_{\xi\in S}\polinf^{\mathrm{in}}(\xi)$. As a result, optimally curating data under our approximate linear model on policy performance simply entails selecting the least influential demonstrations from the training data $\calD$---in the case of data filtering---or selecting the most influential demonstrations from a new set of demonstrations $\calH$---in the case of data selection:

\begin{minipage}{0.47\textwidth}
  \textbf{\cref{task:filter-k}: Filter-$k$ Demonstrations}
  \begin{equation}\label{eq:prune-topk}
    S^\star_{\mathrm{out}} = \argtop\text{-}k\big(\{\polinf^{\mathrm{out}}(\xi^i) : \xi^i \in \mathcal{D}\}\big),
  \end{equation}
\end{minipage}
\hfill
\begin{minipage}{0.47\textwidth}
  \textbf{\cref{task:select-k}: Select-$k$ Demonstrations}
  \begin{equation}\label{eq:select-topk}
    S^\star_{\mathrm{in}} = \argtop\text{-}k\big(\{\polinf^{\mathrm{in}}(\xi^i) : \xi^i \in \mathcal{H}\}\big).
  \end{equation}
\end{minipage}

We note that by linearly approximating policy performance changes using $\polinf$, we construct what is commonly termed a (linear) \emph{datamodel}~\cite{ilyas2022datamodels}. 
As shown in NLP \cite{engstrom2024dsdm}, using such first-order approximations for data curation can often greatly improve model performance over manual notions of quality. 

\subsection{Additional Quality Metrics}\label{sec:methods-quality}
In \cref{sec:performance-influence}, we constructed a method to estimate $\polinf$ from a dataset of policy rollouts $\calD_\tau$ by relying on policy gradient methods.
Therefore, the estimated performance influence $\widehat{\Psi}_{\pi\text{-}\mathrm{inf}}$ becomes increasingly noisy as we reduce the number of rollouts $m$ to evaluate the policy---akin to the high variance problem of the \texttt{REINFORCE} algorithm. 
To complement the analysis in \cref{sec:performance-influence}, we explore the integration of a \emph{reward-agnostic, heuristic} demonstration quality metric based on the action influence scores $\actinf$: 
\begin{equation}\label{eq:quality-influence-fn}\small
    \Psi_{\text{qual}}(\xi; \mathcal{D}_\tau) := \frac{1}{m}\sum_{\tau \in \mathcal{D}_\tau} \max_{\tiny{(s',a') \in \tau}} \min_{(\tiny{s,a) \in \xi}} \actinf\big((s',a'), (s,a)\big) - \min_{\tiny{(s',a') \in \tau}} \max_{\tiny{(s,a) \in \xi}} \actinf\big((s',a'), (s,a)\big).
\end{equation}
We base the quality score \cref{eq:quality-influence-fn} on the intuition that we should penalize demonstrations containing outlier or noisy influence scores \cite[Sec. 5.2]{koh2017understanding}, \cite{hejna2025robotdatacurationmutual}.
As such, we posit that this heuristic can reduce variance on tasks requiring precise motion, yet introduce bias uncorrelated with performance in other settings.
Thus, in \cref{sec:experiments}, we investigate when the quality score can complement $\polinf$ to curate data by taking their convex combination, $\alpha \polinf + (1-\alpha)\Psi_{\mathrm{qual}}$, ablating $\alpha=1$ (\basemethod{}) and $\alpha = 1/2$ (\qualitymethod{}).

\section{Experiments}\label{sec:experiments}
We conduct a series of experiments to test the efficacy of \basemethod{} alongside state-of-the-art baselines for robot data curation. These experiments take place across three simulated tasks from the RoboMimic benchmark suite~\cite{pmlr-v164-mandlekar22a} and three real-world tasks with a Franka FR3 manipulator (see \cref{fig:franka-dp-results}). These tasks comprise a taxonomy of settings where data curation may benefit policy performance. For a detailed description of our tasks, datasets, baselines, evaluation protocol, and hardware setup, please refer to \cref{appx:experiments} 

\textbf{Evaluation.} We study the filter-$k$ (\cref{task:filter-k}) and select-$k$ (\cref{task:select-k}) curation tasks wherever applicable. For statistical significance, we start filter-$k$ and select-$k$ from random $\sim2/3$ and $\sim1/3$ subsets in RoboMimic (300 demonstrations per task total), and random $\sim9/10$ and $\sim4/10$ subsets on Franka tasks (120-160 demonstrations per task total), respectively. We use the official convolutional-based diffusion policy implementation~\cite{chi2023diffusion} for all tasks to measure the effect of curation on a state-of-the-art policy architecture. 
Details on the influence function computation for diffusion models are provided in \cref{appx:method}.
We also consider the official $\pi_0$ implementation~\cite{black2410pi0} for real-world tasks. To reflect practical constraints, we limit the rollout budget (i.e., the number of rollouts in $\mathcal{D}_\tau = \{\tau^i\}_{i=1}^m$ a curation algorithm may use, as described in \cref{sec:formulation}) to $m = 100$ and $m = 25$ for simulated and real-world tasks, respectively. We report policy success rates over 500 rollouts averaged over the last 10 policy checkpoints for simulated tasks, and 25 rollouts performed with the last checkpoint for real-world tasks.

\textbf{Baselines.}
We consider baselines from several methodological categories: DemInf~\cite{hejna2025robotdatacurationmutual}---applicable only to filter-$k$ (\cref{task:filter-k})---curates data offline (i.e., without rollouts) by maximizing mutual information, promoting diverse and predictable demonstrations; Demo-SCORE~\cite{chen2025curating} trains binary classifiers to distinguish states from successful and failed rollouts, retaining demonstrations with a high average state success probability; Success Similarity is a custom method that ranks demonstrations by their average state similarity to successful rollouts; Random chooses demonstrations uniformly at random; Oracle approximates an upper bound on performance by curating data with privileged access to ground-truth demonstration labels, e.g., indicating demonstration quality, strategy robustness, or other properties.

\subsection{Setting 1: Improving Policy Performance in Mixed-Quality Regimes}\label{sec:exp-quality}
We first study curation of mixed-quality datasets, where training on lower-quality demonstrations may degrade policy performance~\cite{pmlr-v164-mandlekar22a, hejna2025robotdatacurationmutual}. 
We use the ``Lift,'' ``Square,'' and ``Transport'' tasks from RoboMimic's multi-human (MH) task suite, which provides ground-truth quality labels for demonstrations. 
The ``Lift'' and ``Square'' tasks contain three quality tiers \{“low”, “medium”, “high”\}, while the more complex bi-manual ``Transport'' task contains six quality tiers \{“low-low”, “low-medium”, \ldots\}. 
On hardware, we design the “Figure-8” task (\cref{fig:franka-dp-results}(a)), where the robot must tie a simplified cleat hitch---a knot that follows a figure-8 pattern---requiring precise manipulation of a deformable rope.

\begin{figure}[t]
    \centering
    \includegraphics[width=\linewidth]{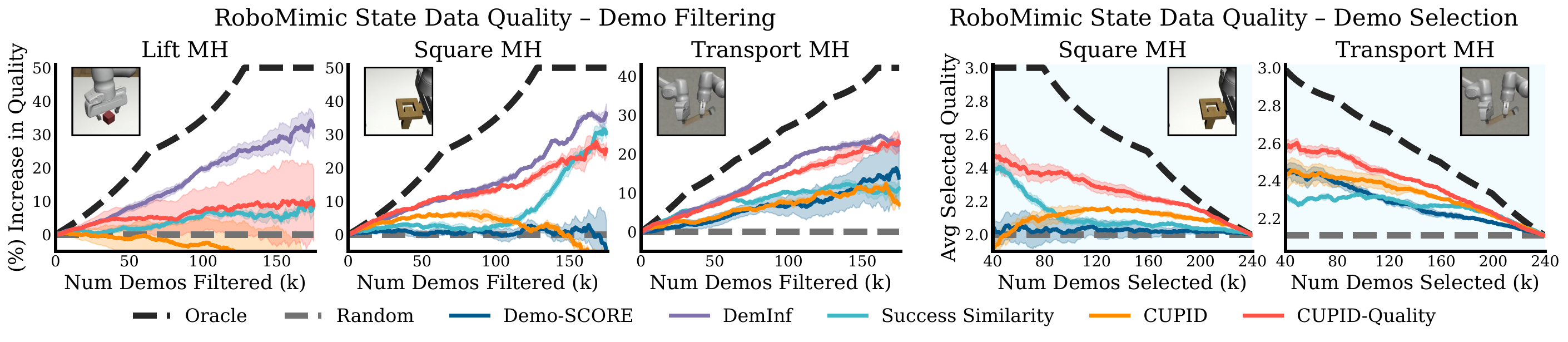}
    \includegraphics[width=\linewidth]{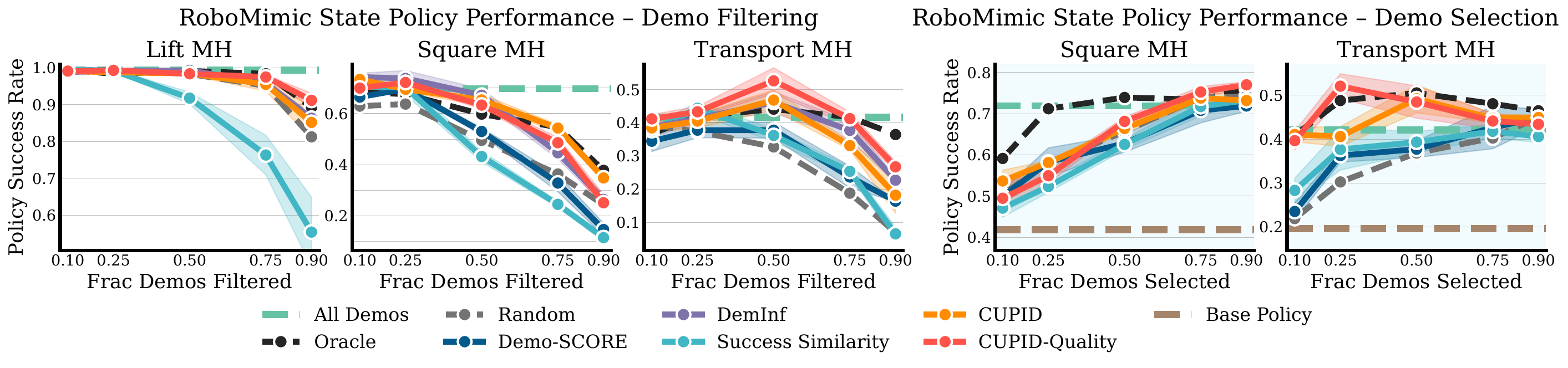}
    \vspace{-20pt}
    \caption{\small
        RoboMimic mixed-quality curation results. \textbf{Top: Data Quality.} Baselines often prioritize demonstration quality (e.g., DemInf~\cite{hejna2025robotdatacurationmutual}), but higher demonstration quality does always translate to higher policy success rates. In contrast, \basemethod{} targets demonstrations that most strongly contribute to downstream policy performance. \textbf{Bottom: Policy Performance.} Diffusion policies trained on data curated by \basemethod{} achieve higher success rates than baselines, despite using demonstrations of perceived lower quality. Although combining performance and quality measures (\qualitymethod{}) yields the best policies on mixed-quality datasets, quality measures can degrade performance in other settings (see \cref{fig:franka-dp-results}).
        Results are averaged over 3 random seeds (500 policies trained across settings). Success rates are computed over 50 rollouts from the last 10 checkpoints (500 rollouts total).
    }
    \label{fig:robomimic-dp-results}
\end{figure}

\textbf{RoboMimic analysis.} \cref{fig:robomimic-dp-results} presents the RoboMimic benchmark results: the top row shows data quality trends for filter-$k$ and select-$k$ across varying $k$, while the bottom row reports success rates of diffusion policies trained on the corresponding curated datasets. As expected, we first observe that DemInf---which targets demonstration quality---curates datasets of the highest overall quality by RoboMimic's ground-truth labels for filter-$k$ (top row, \cref{fig:robomimic-dp-results}). However, policies trained on data curated by \basemethod{} 
consistently match or outperform those of DemInf (bottom row, \cref{fig:robomimic-dp-results}). 
This indicates that human perception of demonstration quality does not necessarily correspond to data that maximizes downstream policy success.
Second, we find the state similarity heuristics employed by Demo-SCORE and Success Similarity to be relatively ineffective in challenging mixed-quality regimes, where successful and failed rollouts exhibit similar states. 
Lastly, \qualitymethod{}, which evenly balances demonstration quality and downstream performance impact (\cref{sec:methods-quality}), attains the highest policy success rates---surpassing the Oracle in 3/5 cases, and achieving an even higher success rate than the official diffusion policy~\cite{chi2023diffusion} on ``Transport MH'' while using fewer than (i) 33\% of the original 300 demonstrations and (ii) 10\% of the model parameters. 
We provide an extended discussion of the RoboMimic results in \cref{appx:results-robomimic-discussion}.

\textbf{Figure-8 analysis.} 
\cref{fig:franka-dp-results}(a) shows diffusion policy results on the real-world ``Figure-8'' task. First, \basemethod{} improves over the base policy's success rate by 38\% (averaged over filtering and selection). Second, as in RoboMimic, \qualitymethod{} further strengthens curation performance, corroborating the utility of quality metrics (\cref{eq:quality-influence-fn}) in mixed-quality regimes. 
As shown in \cref{fig:franka-dp-distr-filter-dataset-results}(a) (filtering; see \cref{appx:select-distr} for selection), both \basemethod{} and \qualitymethod{} successfully retain high-quality demonstrations, whereas baselines such as Demo-SCORE discard some in favor of lower-quality demonstrations. Overall, training on lower-quality demonstrations appears to adversely affect policy performance on the ``Figure-8'' task.

\begin{figure}[t]
    \centering
    \includegraphics[width=\linewidth]{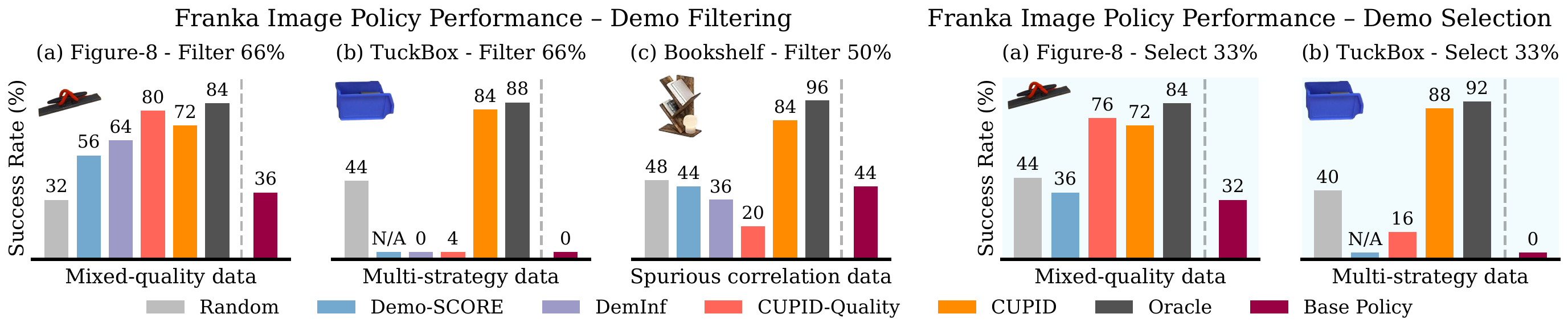}
    \includegraphics[width=\linewidth]{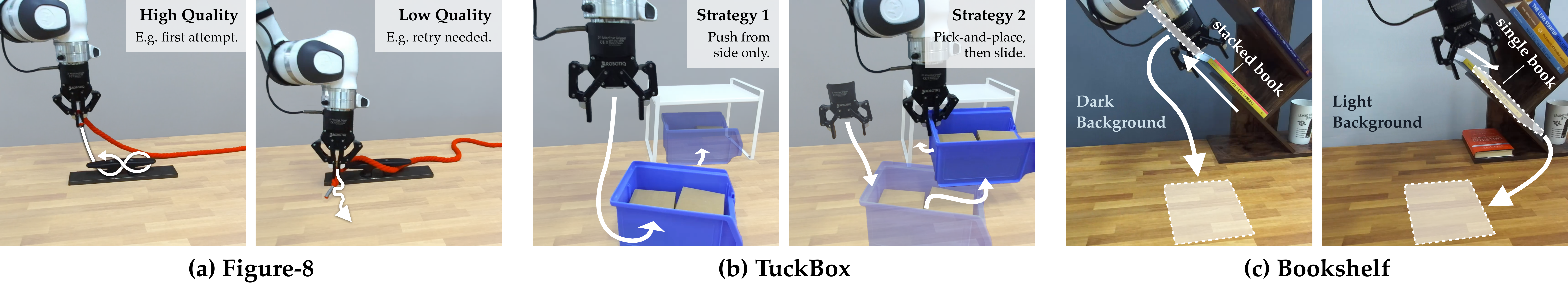}
    \vspace{-18pt}
    \caption{\small
        \textbf{Franka real-world diffusion policy performance.} \basemethod{}, which curates demonstrations \textit{w.r.t.} policy performance, improves success rates on mixed-quality datasets, identifies robust strategies, and disentangles spurious correlations that hinder performance. Although quality measures (e.g., DemInf, \qualitymethod{}) help in mixed-quality settings (Figure-8; \cref{fig:robomimic-dp-results}), they degrade performance when higher-quality demonstrations induce brittle strategies at test time (TuckBox), or when quality is not the primary factor limiting policy success (Bookshelf). Overall, curating data based on performance (\basemethod{}) maintains robustness across these settings. 
    }
    \label{fig:franka-dp-results}
\end{figure}

\subsection{Setting 2: Identifying Robust Test-time Strategies from Policy Failures}\label{sec:exp-strategies}
Heterogeneous imitation learning datasets may contain multiple strategies for solving a task, some of which can fail under distribution shifts at deployment. We design a real-world ``TuckBox'' task, where a robot must tuck a recycling bin under a receptacle by (i) sliding or (ii) first repositioning it via pick-and-place (see \cref{fig:franka-dp-results}(b)). The dataset contains a 2:1 ratio of sliding to pick-and-place demonstrations, making sliding the dominant strategy. At test time, we induce an imperceptible distribution shift by altering the bin's mass distribution, rendering sliding unreliable. 
In this setting, curation aims to rebalance the dataset to promote strategies that are more robust to unforeseen shifts at deployment.

\textbf{TuckBox analysis.} \cref{fig:franka-dp-results}(b) shows the diffusion policy results on ``TuckBox.'' Due to the strategy imbalance, the base policy exclusively exhibits the sliding behavior, resulting in a 100\% failure rate under the distribution shift. This immediately invalidates the use of Demo-SCORE, which requires both successful and failed rollouts. In contrast, \basemethod{} does not require observing successes: by linking failures to the demonstrations that influenced them, curating with \basemethod{} yields a policy that exhibits increased pick-and-place behavior, performing comparably (84\%-88\% success rate) to the Oracle. 
In contrast, both DemInf and \qualitymethod{} incorrectly associate the higher-variance pick-and-place demonstrations with lower quality, resulting in more uniform filtering across strategies (see \cref{fig:franka-dp-distr-filter-dataset-results}(b)). As a result, policies trained on data curated by these baselines default to the brittle sliding strategy at deployment.

\begin{figure}[t]
    \centering
    
    \begin{subfigure}[b]{\linewidth}
        \centering
        \includegraphics[width=0.95\linewidth]{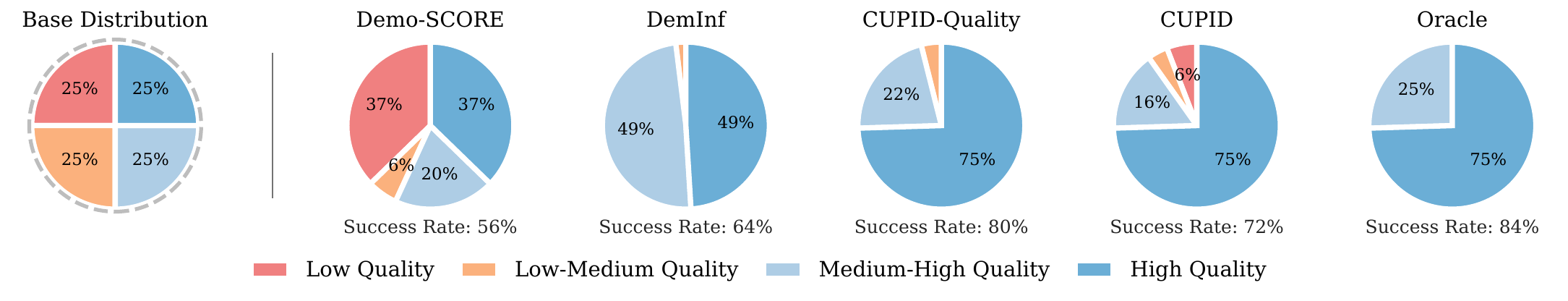}
        \vspace{-6pt}
        \caption{\footnotesize 
            \textbf{Figure-8:} Distribution of curated demonstrations after \textit{filtering} 66\%. Higher-quality demos are better.
        }
        \label{fig:franka-dp-distr-filter-dataset-results-figure8}
    \end{subfigure}

    \vspace{18pt}

    \begin{subfigure}[b]{\linewidth}
        \centering
        \includegraphics[width=0.95\linewidth]{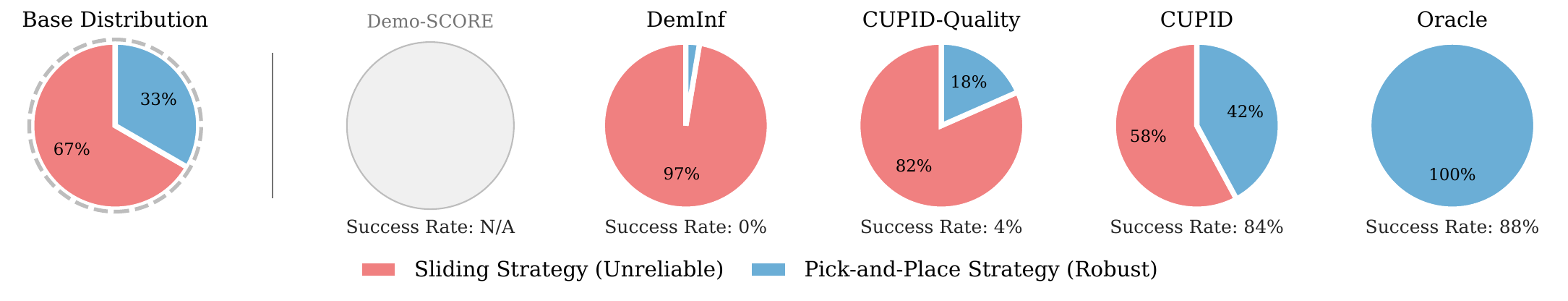}
        \vspace{-6pt}
        \caption{\footnotesize
            \textbf{TuckBox:} Distribution of curated demonstrations after \textit{filtering} 66\%. Pick-and-place demos are better.
        }
        \label{fig:franka-dp-distr-filter-dataset-results-tuckbox}
    \end{subfigure}
    
    \vspace{18pt}
    
    \begin{subfigure}[b]{\linewidth}
        \centering
        \includegraphics[width=0.95\linewidth]{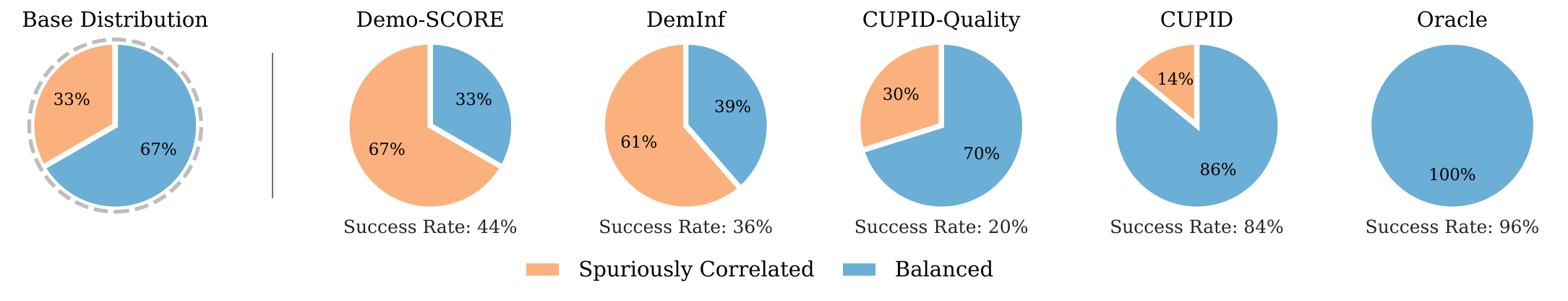}
        \vspace{-6pt}
        \caption{\footnotesize
            \textbf{Bookshelf:} Distribution of curated demonstrations after \textit{filtering} 50\%. Balanced data is better.
        }
        \label{fig:franka-dp-distr-filter-dataset-results-bookshelf}
    \end{subfigure}
    
    \vspace{8pt}   
    
    \caption{\small
        \textbf{Franka diffusion policy curated dataset distributions for filtering (\cref{task:filter-k}).} \basemethod{} filters out lower-quality demonstrations (Figure-8), brittle strategies (TuckBox), and spuriously correlated examples (Bookshelf), improving policy performance across tasks. While curation heuristics employed by baselines may be effective in some cases (e.g., DemInf and \qualitymethod{} in Figure-8), they can lead to suboptimal pruning in others. \\
    }
    \label{fig:franka-dp-distr-filter-dataset-results}
\end{figure}

\subsection{Setting 3: Disentangling Spurious Correlations in Demonstration Data}\label{sec:exp-correlations}
Spurious correlations in training data may cause a policy to rely on non-causal features, hindering generalization to variations in the input or task~\cite{de2019causal}. We design a real-world ``Bookshelf'' task, where a robot must extract a target book via (i) horizontal or (ii) vertical pulling motion, depending on whether another book is stacked above the target book. While both strategies are equally represented in the training set, each co-occurs more frequently with a certain background color (see \cref{fig:franka-dp-results}(c)).
At evaluation, we test the policy under slight variations in the number and position of distractor books, while keeping the white background fixed---the correlate associated with the horizontal pulling behavior.

\textbf{Bookshelf analysis.} Diffusion policy results are shown in \cref{fig:franka-dp-results}(c). 
The base policy achieves only a 44\% success rate, as the presence of the white background often causes the policy to extract the target book horizontally despite another book being stacked on top (causing it to fall). 
Interestingly, by training classifiers to distinguish failed from successful states, Demo-SCORE appears to misattribute failure to the presence of rollout correlates (the stacked book) rather than causal factors (the white background). 
In contrast, \basemethod{} attains an 84\% success rate by identifying demonstrations that causally drive failure---in this case, horizontal pulling motion with a white background---enabling dataset rebalancing that mitigates the effect of spurious correlations (see \cref{fig:franka-dp-distr-filter-dataset-results}(c)). 
As in \cref{sec:exp-strategies}, DemInf and \qualitymethod{} incorrectly prioritize the lower-variance horizontal pulling motion, yielding negligible performance gains.

\ifarxiv
\section{Discussion \& Ablations}\label{sec:discussion}
\subsection{How is curation performance affected by properties of the data and the task?}\label{sec:discussion-properties}

Our mixed-quality curation experiments (\cref{fig:robomimic-dp-results} and \cref{fig:franka-dp-results}(a)) reveal that while curation strengthens performance on ``Transport MH'' and ``Figure-8'' (i.e., a fraction of the demonstrations harm policy performance), removing almost \textit{any} demonstration degrades performance on ``Square MH'' (i.e., all demonstrations appear important). 
In contrast, only about 15\% of the dataset is necessary to maximize performance on ``Lift MH''  (i.e., the dataset is highly redundant)\footnote{Note that \cref{fig:robomimic-dp-results} does not include select-$k$ curation results for ``Lift MH'' because the base policy already achieves a 100\% success rate, leaving no further room for improvement by selecting additional demonstrations.}. 
These results indicate that the potential benefits of data curation depend on properties of both the data and the task. 
For example, one possible hypothesis is that curation is most effective in complex, precision-critical settings (e.g., ``Transport MH''), whereas for tasks with greater tolerance for error (e.g., ``Lift MH''), state-of-the-art policies~\cite{chi2023diffusion} appear less sensitive to---and may even benefit from---training on lower-quality demonstrations.

\subsection{How many policy rollouts are required for effective curation with \basemethod{}?}\label{sec:discussion-rollouts}

\basemethod{} uses a \texttt{REINFORCE}-style estimator to compute the performance influence of each demonstration (\cref{eq:policy-inf-deriv}) for curation.
Thus, the accuracy of estimated performance influences depends on the number of policy rollouts.
While \texttt{REINFORCE}~\cite{sutton1999policy} often yields high-variance gradient estimates under limited rollout budgets, e.g., in reinforcement learning contexts~\cite{greensmith2004variance}, we highlight that our curation objective imposes a lower fidelity requirement: since curation with \basemethod{} involves top-$k$ selection (\cref{sec:methods-curation}), it suffices to rank helpful demonstrations above harmful ones (requiring fewer rollouts) rather than to estimate performance influence precisely (requiring many rollouts). 
As shown in \cref{fig:robomimic-state-data-quality-cupid-rollout}, the ranking of demonstrations stabilizes with approximately $m \in [25, 50]$ rollouts on ``Lift MH'' and ``Square MH,'' and $m \in [50, 100]$ rollouts on ``Transport MH.'' Similarly, we use only $m = 25$ rollouts for our real-world Franka tasks (\cref{fig:franka-dp-results}). These results support the practicality of \basemethod{} under realistic rollout budgets, while noting that more complex tasks (e.g., ``Transport MH'') may benefit from a greater number of rollouts.

\begin{figure}[H]
    \centering
    \includegraphics[width=\linewidth]{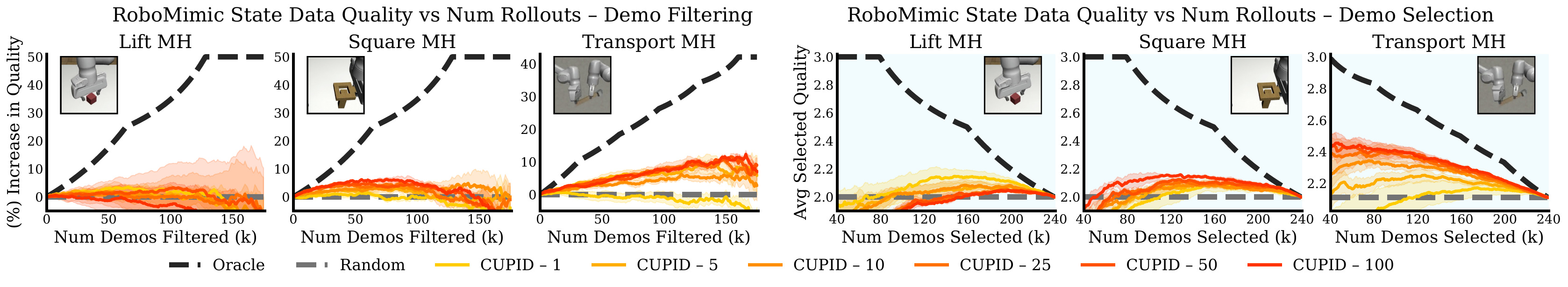}
    \vspace{-20pt}
    \caption{\small
        \textbf{\basemethod{} ablation on the number of policy rollouts.} 
        Performance influences (\cref{eq:policy-inf-deriv}) converge with $m \in [25, 50]$ rollouts on ``Lift MH'' and ``Square MH'' (yielding similar quality trends), and $m \in [50, 100]$ rollouts on ``Transport MH,'' validating the practical applicability of \basemethod{} under realistic rollout budgets. 
        Curation is performed on diffusion policies.
        Results are averaged over 3 random seeds.
        Errors bars represent standard error.
    }
    \label{fig:robomimic-state-data-quality-cupid-rollout}
\end{figure}

\subsection{Can data curated for single-task policies strengthen generalist policy performance?}\label{sec:discussion-pi0-transfer}

\begin{wrapfigure}{rt}{0.39\textwidth}
    \centering
    \vspace{-14pt}
    \includegraphics[width=1.0\textwidth]{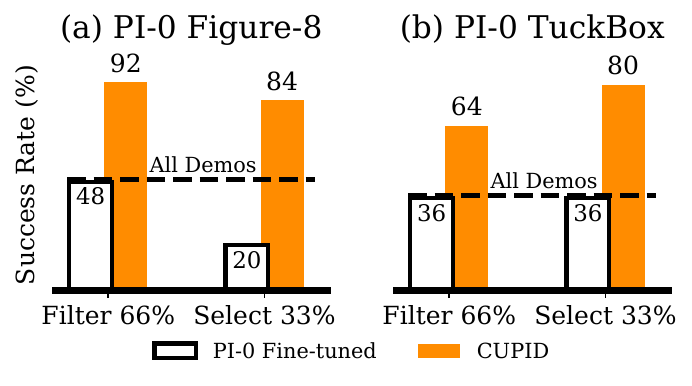}
    \caption{\small 
        Data curated for single-task diffusion policies improves $\pi_0$~\cite{black2410pi0} post-training performance. 
        Additional results in \cref{fig:franka-pi0-results-appx}.
    }
    \vspace{-12pt}
    \label{fig:franka-pi0-transfer-results}
\end{wrapfigure}

In \cref{fig:franka-pi0-transfer-results}, we show that datasets curated by \basemethod{} for single-task diffusion policies can significantly improve the fine-tuned performance of a generalist Vision-Language-Action (VLA) model, $\pi_0$~\cite{black2410pi0}.
While \basemethod{} is, in principle, tailored to the specific policy used during rollouts, it consistently identifies low-quality, stochastic behaviors in ``Figure-8'' and unreliable strategies in ``TuckBox'' (\cref{fig:franka-dp-distr-filter-dataset-results})---both intrinsic properties of the data. 
Filtering these poorer demonstrations (or selecting better ones) is thereby likely to improve the performance \textit{any} policy.
This highlights a promising direction to alleviate the computational cost of \basemethod{} in large-scale settings: use smaller, single-task policies to curate datasets for larger, multi-task models. 

\textit{VLA Robustness.} \cref{fig:franka-pi0-transfer-results} also suggests that scaling the pre-training of VLA models does not inherently enable them to leverage their generalist knowledge to, e.g., \textit{ignore} low-quality behaviors or brittle strategies in demonstration data. That is, data curation still appears important for VLA post-training.

\fi

\section{Conclusion}\label{sec:conclusion}
In this work, we study the problem of data curation for robot imitation learning. We present \basemethod{}, a novel data curation method that uses influence functions to measure the causal impact of a demonstration on the policy's closed-loop performance. Our results highlight the general utility of performance-based curation for two key curation tasks---filtering existing training demonstrations and subselecting new demonstrations---and across diverse curation settings, where a policy's test-time performance varies with the choice of training data. 
\ifarxiv
Among many key problems in robotics, it is inherently difficult to develop strong intuitions about how training data influences downstream policy behavior, and to delineate why a policy trained exclusively on expert demonstration data would exhibit suboptimal performance at deployment. We hope this work spurs continued interest in pursuit of these questions.
\fi

\section{Limitations}\label{sec:limitations}
\textbf{Curation tasks.} The curation tasks considered in this work (\cref{task:filter-k} and \cref{task:select-k}) aim to curate performance-maximizing datasets for a specified filtering or selection quantity of demonstrations $k$. Determining the suitable quantity of demonstrations to curate represents a possible point of extension.

\textbf{Data properties.} Critically, future work should further investigate how properties of the data dictate the extent to which curation can improve policy performance, as discussed in \cref{sec:discussion-properties}. 

\textbf{Data explainability.} Our methods focus on curating existing demonstrations as a first step. However, future work may seek to interpret the properties of influential demonstrations to actively inform subsequent data collection efforts---for example, by providing instructions to data collectors. 

\textbf{Selection methods.} While the \emph{greedy} selection procedures used in \cref{eq:prune-topk} and \cref{eq:select-topk} are tractable to optimize and often improve over quality- and similarity-based measures~\cite{engstrom2024dsdm}, they ignore the interactions between demonstrations in the curated set~\cite{koh2019accuracy, ilyas2022datamodels}. This can temper performance gains when the size of the curated set is large. Future work should investigate higher-order approximations that consider the joint diversity of the curated dataset, as is common in the active learning literature (e.g., \cite[Sec. 4.3]{Settles2012}).

\textbf{Larger datasets.} Estimating performance influences over the full demonstration dataset incurs a computational cost comparable to that of policy training. Reducing this expense in large-scale settings is an important future direction. For example, one could approximate group effects~\cite{koh2019accuracy} via random sampling or limit influence estimation to smaller data subsets identified using coarse-grained heuristics. 

\textbf{Estimator variance.} Finally, although we observe stable performance from \basemethod{} across curation settings, the use of the \texttt{REINFORCE} estimator may result in high variance influence scores, e.g., when the number of policy rollouts is small. In such settings, variance reduction techniques, such as those typically used in reinforcement learning~\cite{greensmith2004variance}, may further improve the fidelity of our influence scores.

\acknowledgments{
    The authors would like to thank \href{https://sungminpark.com/}{Sam Park} and \href{https://andrewilyas.com/}{Andrew Ilyas} for the helpful conversations on data attribution, and \href{https://joeyhejna.com/}{Joey Hejna} for early-staged technical feedback on the project. Toyota Research Institute provided funds to support this work. This work was also supported by the DARPA TIAMAT program. 
}

\bibliography{references}  %

\begin{thebibliography}{66}
\providecommand{\natexlab}[1]{#1}
\providecommand{\url}[1]{\texttt{#1}}
\expandafter\ifx\csname urlstyle\endcsname\relax
  \providecommand{\doi}[1]{doi: #1}\else
  \providecommand{\doi}{doi: \begingroup \urlstyle{rm}\Url}\fi

\bibitem[Grosse et~al.(2023)Grosse, Bae, Anil, Elhage, Tamkin, Tajdini, Steiner, Li, Durmus, Perez, et~al.]{grosse2023studying}
R.~Grosse, J.~Bae, C.~Anil, N.~Elhage, A.~Tamkin, A.~Tajdini, B.~Steiner, D.~Li, E.~Durmus, E.~Perez, et~al.
\newblock Studying large language model generalization with influence functions.
\newblock \emph{arXiv preprint arXiv:2308.03296}, 2023.

\bibitem[Park et~al.(2023)Park, Georgiev, Ilyas, Leclerc, and Madry]{park2023trak}
S.~M. Park, K.~Georgiev, A.~Ilyas, G.~Leclerc, and A.~Madry.
\newblock Trak: Attributing model behavior at scale.
\newblock In \emph{International Conference on Machine Learning}, pages 27074--27113. PMLR, 2023.

\bibitem[Engstrom et~al.(2024)Engstrom, Feldmann, and Madry]{engstrom2024dsdm}
L.~Engstrom, A.~Feldmann, and A.~Madry.
\newblock Dsdm: Model-aware dataset selection with datamodels.
\newblock In \emph{International Conference on Machine Learning}, pages 12491--12526. PMLR, 2024.

\bibitem[Lee et~al.(2022)Lee, Ippolito, Nystrom, Zhang, Eck, Callison-Burch, and Carlini]{lee2021deduplicating}
K.~Lee, D.~Ippolito, A.~Nystrom, C.~Zhang, D.~Eck, C.~Callison-Burch, and N.~Carlini.
\newblock Deduplicating training data makes language models better.
\newblock In S.~Muresan, P.~Nakov, and A.~Villavicencio, editors, \emph{Proceedings of the 60th Annual Meeting of the Association for Computational Linguistics (Volume 1: Long Papers)}, pages 8424--8445, Dublin, Ireland, May 2022. Association for Computational Linguistics.
\newblock \doi{10.18653/v1/2022.acl-long.577}.
\newblock URL \url{https://aclanthology.org/2022.acl-long.577/}.

\bibitem[Tirumala et~al.(2023)Tirumala, Simig, Aghajanyan, and Morcos]{tirumala2023d4}
K.~Tirumala, D.~Simig, A.~Aghajanyan, and A.~Morcos.
\newblock D4: Improving llm pretraining via document de-duplication and diversification.
\newblock \emph{Advances in Neural Information Processing Systems}, 36:\penalty0 53983--53995, 2023.

\bibitem[Albalak et~al.(2024)Albalak, Elazar, Xie, Longpre, Lambert, Wang, Muennighoff, Hou, Pan, Jeong, Raffel, Chang, Hashimoto, and Wang]{albalak2024survey}
A.~Albalak, Y.~Elazar, S.~M. Xie, S.~Longpre, N.~Lambert, X.~Wang, N.~Muennighoff, B.~Hou, L.~Pan, H.~Jeong, C.~Raffel, S.~Chang, T.~Hashimoto, and W.~Y. Wang.
\newblock A survey on data selection for language models.
\newblock \emph{Transactions on Machine Learning Research}, 2024.
\newblock ISSN 2835-8856.
\newblock URL \url{https://openreview.net/forum?id=XfHWcNTSHp}.

\bibitem[O’Neill et~al.(2024)O’Neill, Rehman, Maddukuri, Gupta, Padalkar, Lee, Pooley, Gupta, Mandlekar, Jain, Tung, Bewley, Herzog, Irpan, Khazatsky, Rai, Gupta, Wang, Singh, Garg, Kembhavi, Xie, Brohan, Raffin, Sharma, Yavary, Jain, Balakrishna, Wahid, Burgess-Limerick, Kim, Schölkopf, Wulfe, Ichter, Lu, Xu, Le, Finn, Wang, Xu, Chi, Huang, Chan, Agia, Pan, Fu, Devin, Xu, Morton, Driess, Chen, Pathak, Shah, Büchler, Jayaraman, Kalashnikov, Sadigh, Johns, Foster, Liu, Ceola, Xia, Zhao, Stulp, Zhou, Sukhatme, Salhotra, Yan, Feng, Schiavi, Berseth, Kahn, Wang, Su, Fang, Shi, Bao, Ben~Amor, Christensen, Furuta, Walke, Fang, Ha, Mordatch, Radosavovic, Leal, Liang, Abou-Chakra, Kim, Drake, Peters, Schneider, Hsu, Bohg, Bingham, Wu, Gao, Hu, Wu, Wu, Sun, Luo, Gu, Tan, Oh, Wu, Lu, Yang, Malik, Silvério, Hejna, Booher, Tompson, Yang, Salvador, Lim, Han, Wang, Rao, Pertsch, Hausman, Go, Gopalakrishnan, Goldberg, Byrne, Oslund, Kawaharazuka, Black, Lin, Zhang, Ehsani, Lekkala, Ellis, Rana, Srinivasan, Fang,
  Singh, Zeng, Hatch, Hsu, Itti, Chen, Pinto, Fei-Fei, Tan, Fan, Ott, Lee, Weihs, Chen, Lepert, Memmel, Tomizuka, Itkina, Castro, Spero, Du, Ahn, Yip, Zhang, Ding, Heo, Srirama, Sharma, Kim, Kanazawa, Hansen, Heess, Joshi, Suenderhauf, Liu, Di~Palo, Shafiullah, Mees, Kroemer, Bastani, Sanketi, Miller, Yin, Wohlhart, Xu, Fagan, Mitrano, Sermanet, Abbeel, Sundaresan, Chen, Vuong, Rafailov, Tian, Doshi, Martín-Martín, Baijal, Scalise, Hendrix, Lin, Qian, Zhang, Mendonca, Shah, Hoque, Julian, Bustamante, Kirmani, Levine, Lin, Moore, Bahl, Dass, Sonawani, Song, Xu, Haldar, Karamcheti, Adebola, Guist, Nasiriany, Schaal, Welker, Tian, Ramamoorthy, Dasari, Belkhale, Park, Nair, Mirchandani, Osa, Gupta, Harada, Matsushima, Xiao, Kollar, Yu, Ding, Davchev, Zhao, Armstrong, Darrell, Chung, Jain, Vanhoucke, Zhan, Zhou, Burgard, Chen, Wang, Zhu, Geng, Liu, Liangwei, Li, Lu, Ma, Kim, Chebotar, Zhou, Zhu, Wu, Xu, Wang, Bisk, Cho, Lee, Cui, Cao, Wu, Tang, Zhu, Zhang, Jiang, Li, Li, Iwasawa, Matsuo, Ma, Xu, Cui, Zhang, and
  Lin]{o2024open}
A.~O’Neill, A.~Rehman, A.~Maddukuri, A.~Gupta, A.~Padalkar, A.~Lee, A.~Pooley, A.~Gupta, A.~Mandlekar, A.~Jain, A.~Tung, A.~Bewley, A.~Herzog, A.~Irpan, A.~Khazatsky, A.~Rai, A.~Gupta, A.~Wang, A.~Singh, A.~Garg, A.~Kembhavi, A.~Xie, A.~Brohan, A.~Raffin, A.~Sharma, A.~Yavary, A.~Jain, A.~Balakrishna, A.~Wahid, B.~Burgess-Limerick, B.~Kim, B.~Schölkopf, B.~Wulfe, B.~Ichter, C.~Lu, C.~Xu, C.~Le, C.~Finn, C.~Wang, C.~Xu, C.~Chi, C.~Huang, C.~Chan, C.~Agia, C.~Pan, C.~Fu, C.~Devin, D.~Xu, D.~Morton, D.~Driess, D.~Chen, D.~Pathak, D.~Shah, D.~Büchler, D.~Jayaraman, D.~Kalashnikov, D.~Sadigh, E.~Johns, E.~Foster, F.~Liu, F.~Ceola, F.~Xia, F.~Zhao, F.~Stulp, G.~Zhou, G.~S. Sukhatme, G.~Salhotra, G.~Yan, G.~Feng, G.~Schiavi, G.~Berseth, G.~Kahn, G.~Wang, H.~Su, H.-S. Fang, H.~Shi, H.~Bao, H.~Ben~Amor, H.~I. Christensen, H.~Furuta, H.~Walke, H.~Fang, H.~Ha, I.~Mordatch, I.~Radosavovic, I.~Leal, J.~Liang, J.~Abou-Chakra, J.~Kim, J.~Drake, J.~Peters, J.~Schneider, J.~Hsu, J.~Bohg, J.~Bingham, J.~Wu, J.~Gao, J.~Hu,
  J.~Wu, J.~Wu, J.~Sun, J.~Luo, J.~Gu, J.~Tan, J.~Oh, J.~Wu, J.~Lu, J.~Yang, J.~Malik, J.~Silvério, J.~Hejna, J.~Booher, J.~Tompson, J.~Yang, J.~Salvador, J.~J. Lim, J.~Han, K.~Wang, K.~Rao, K.~Pertsch, K.~Hausman, K.~Go, K.~Gopalakrishnan, K.~Goldberg, K.~Byrne, K.~Oslund, K.~Kawaharazuka, K.~Black, K.~Lin, K.~Zhang, K.~Ehsani, K.~Lekkala, K.~Ellis, K.~Rana, K.~Srinivasan, K.~Fang, K.~P. Singh, K.-H. Zeng, K.~Hatch, K.~Hsu, L.~Itti, L.~Y. Chen, L.~Pinto, L.~Fei-Fei, L.~Tan, L.~J. Fan, L.~Ott, L.~Lee, L.~Weihs, M.~Chen, M.~Lepert, M.~Memmel, M.~Tomizuka, M.~Itkina, M.~G. Castro, M.~Spero, M.~Du, M.~Ahn, M.~C. Yip, M.~Zhang, M.~Ding, M.~Heo, M.~K. Srirama, M.~Sharma, M.~J. Kim, N.~Kanazawa, N.~Hansen, N.~Heess, N.~J. Joshi, N.~Suenderhauf, N.~Liu, N.~Di~Palo, N.~M.~M. Shafiullah, O.~Mees, O.~Kroemer, O.~Bastani, P.~R. Sanketi, P.~T. Miller, P.~Yin, P.~Wohlhart, P.~Xu, P.~D. Fagan, P.~Mitrano, P.~Sermanet, P.~Abbeel, P.~Sundaresan, Q.~Chen, Q.~Vuong, R.~Rafailov, R.~Tian, R.~Doshi, R.~Martín-Martín,
  R.~Baijal, R.~Scalise, R.~Hendrix, R.~Lin, R.~Qian, R.~Zhang, R.~Mendonca, R.~Shah, R.~Hoque, R.~Julian, S.~Bustamante, S.~Kirmani, S.~Levine, S.~Lin, S.~Moore, S.~Bahl, S.~Dass, S.~Sonawani, S.~Song, S.~Xu, S.~Haldar, S.~Karamcheti, S.~Adebola, S.~Guist, S.~Nasiriany, S.~Schaal, S.~Welker, S.~Tian, S.~Ramamoorthy, S.~Dasari, S.~Belkhale, S.~Park, S.~Nair, S.~Mirchandani, T.~Osa, T.~Gupta, T.~Harada, T.~Matsushima, T.~Xiao, T.~Kollar, T.~Yu, T.~Ding, T.~Davchev, T.~Z. Zhao, T.~Armstrong, T.~Darrell, T.~Chung, V.~Jain, V.~Vanhoucke, W.~Zhan, W.~Zhou, W.~Burgard, X.~Chen, X.~Wang, X.~Zhu, X.~Geng, X.~Liu, X.~Liangwei, X.~Li, Y.~Lu, Y.~J. Ma, Y.~Kim, Y.~Chebotar, Y.~Zhou, Y.~Zhu, Y.~Wu, Y.~Xu, Y.~Wang, Y.~Bisk, Y.~Cho, Y.~Lee, Y.~Cui, Y.~Cao, Y.-H. Wu, Y.~Tang, Y.~Zhu, Y.~Zhang, Y.~Jiang, Y.~Li, Y.~Li, Y.~Iwasawa, Y.~Matsuo, Z.~Ma, Z.~Xu, Z.~J. Cui, Z.~Zhang, and Z.~Lin.
\newblock Open x-embodiment: Robotic learning datasets and rt-x models : Open x-embodiment collaboration0.
\newblock In \emph{2024 IEEE International Conference on Robotics and Automation (ICRA)}, pages 6892--6903, 2024.
\newblock \doi{10.1109/ICRA57147.2024.10611477}.

\bibitem[Khazatsky et~al.(2024)Khazatsky, Pertsch, Nair, Balakrishna, Dasari, Karamcheti, Nasiriany, Srirama, Chen, Ellis, Fagan, Hejna, Itkina, Lepert, Ma, Miller, Wu, Belkhale, Dass, Ha, Jain, Lee, Lee, Memmel, Park, Radosavovic, Wang, Zhan, Black, Chi, Hatch, Lin, Lu, Mercat, Rehman, Sanketi, Sharma, Simpson, Vuong, Walke, Wulfe, Xiao, Yang, Yavary, Zhao, Agia, Baijal, Castro, Chen, Chen, Chung, Drake, Foster, Gao, Herrera, Heo, Hsu, Hu, Jackson, Le, Li, Lin, Ma, Maddukuri, Mirchandani, Morton, Nguyen, O'Neill, Scalise, Seale, Son, Tian, Tran, Wang, Wu, Xie, Yang, Yin, Zhang, Bastani, Berseth, Bohg, Goldberg, Gupta, Gupta, Jayaraman, Lim, Malik, Martín-Martín, Ramamoorthy, Sadigh, Song, Wu, Yip, Zhu, Kollar, Levine, and Finn]{khazatsky2024droid}
A.~Khazatsky, K.~Pertsch, S.~Nair, A.~Balakrishna, S.~Dasari, S.~Karamcheti, S.~Nasiriany, M.~K. Srirama, L.~Y. Chen, K.~Ellis, P.~D. Fagan, J.~Hejna, M.~Itkina, M.~Lepert, Y.~J. Ma, P.~T. Miller, J.~Wu, S.~Belkhale, S.~Dass, H.~Ha, A.~Jain, A.~Lee, Y.~Lee, M.~Memmel, S.~Park, I.~Radosavovic, K.~Wang, A.~Zhan, K.~Black, C.~Chi, K.~B. Hatch, S.~Lin, J.~Lu, J.~Mercat, A.~Rehman, P.~R. Sanketi, A.~Sharma, C.~Simpson, Q.~Vuong, H.~R. Walke, B.~Wulfe, T.~Xiao, J.~H. Yang, A.~Yavary, T.~Z. Zhao, C.~Agia, R.~Baijal, M.~G. Castro, D.~Chen, Q.~Chen, T.~Chung, J.~Drake, E.~P. Foster, J.~Gao, D.~A. Herrera, M.~Heo, K.~Hsu, J.~Hu, D.~Jackson, C.~Le, Y.~Li, R.~Lin, Z.~Ma, A.~Maddukuri, S.~Mirchandani, D.~Morton, T.~Nguyen, A.~O'Neill, R.~Scalise, D.~Seale, V.~Son, S.~Tian, E.~Tran, A.~E. Wang, Y.~Wu, A.~Xie, J.~Yang, P.~Yin, Y.~Zhang, O.~Bastani, G.~Berseth, J.~Bohg, K.~Goldberg, A.~Gupta, A.~Gupta, D.~Jayaraman, J.~J. Lim, J.~Malik, R.~Martín-Martín, S.~Ramamoorthy, D.~Sadigh, S.~Song, J.~Wu, M.~C. Yip, Y.~Zhu,
  T.~Kollar, S.~Levine, and C.~Finn.
\newblock {DROID: A Large-Scale In-The-Wild Robot Manipulation Dataset}.
\newblock In \emph{Proceedings of Robotics: Science and Systems}, Delft, Netherlands, July 2024.
\newblock \doi{10.15607/RSS.2024.XX.120}.

\bibitem[Kuhar et~al.(2023)Kuhar, Cheng, Chopra, Bronars, and Xu]{kuhar2023learning}
S.~Kuhar, S.~Cheng, S.~Chopra, M.~Bronars, and D.~Xu.
\newblock Learning to discern: Imitating heterogeneous human demonstrations with preference and representation learning.
\newblock In J.~Tan, M.~Toussaint, and K.~Darvish, editors, \emph{Proceedings of The 7th Conference on Robot Learning}, volume 229 of \emph{Proceedings of Machine Learning Research}, pages 1437--1449. PMLR, 06--09 Nov 2023.

\bibitem[Hejna et~al.(2025)Hejna, Mirchandani, Balakrishna, Xie, Wahid, Tompson, Sanketi, Shah, Devin, and Sadigh]{hejna2025robotdatacurationmutual}
J.~Hejna, S.~Mirchandani, A.~Balakrishna, A.~Xie, A.~Wahid, J.~Tompson, P.~Sanketi, D.~Shah, C.~Devin, and D.~Sadigh.
\newblock Robot data curation with mutual information estimators.
\newblock \emph{arXiv preprint arXiv:2502.08623}, 2025.

\bibitem[Chen et~al.(2025)Chen, Lessing, Liu, and Finn]{chen2025curating}
A.~S. Chen, A.~M. Lessing, Y.~Liu, and C.~Finn.
\newblock Curating demonstrations using online experience.
\newblock \emph{arXiv preprint arXiv:2503.03707}, 2025.

\bibitem[De~Haan et~al.(2019)De~Haan, Jayaraman, and Levine]{de2019causal}
P.~De~Haan, D.~Jayaraman, and S.~Levine.
\newblock Causal confusion in imitation learning.
\newblock \emph{Advances in neural information processing systems}, 32, 2019.

\bibitem[Koh and Liang(2017)]{koh2017understanding}
P.~W. Koh and P.~Liang.
\newblock Understanding black-box predictions via influence functions.
\newblock In \emph{International conference on machine learning}, pages 1885--1894. PMLR, 2017.

\bibitem[Koh et~al.(2019)Koh, Ang, Teo, and Liang]{koh2019accuracy}
P.~W.~W. Koh, K.-S. Ang, H.~Teo, and P.~S. Liang.
\newblock On the accuracy of influence functions for measuring group effects.
\newblock \emph{Advances in neural information processing systems}, 32, 2019.

\bibitem[Hammoudeh and Lowd(2024)]{hammoudeh2024training}
Z.~Hammoudeh and D.~Lowd.
\newblock Training data influence analysis and estimation: A survey.
\newblock \emph{Machine Learning}, 113\penalty0 (5):\penalty0 2351--2403, 2024.

\bibitem[Sutton et~al.(1999)Sutton, McAllester, Singh, and Mansour]{sutton1999policy}
R.~S. Sutton, D.~McAllester, S.~Singh, and Y.~Mansour.
\newblock Policy gradient methods for reinforcement learning with function approximation.
\newblock \emph{Advances in neural information processing systems}, 12, 1999.

\bibitem[Brohan et~al.(2023)Brohan, Brown, Carbajal, Chebotar, Dabis, Finn, Gopalakrishnan, Hausman, Herzog, Hsu, Ibarz, Ichter, Irpan, Jackson, Jesmonth, Joshi, Julian, Kalashnikov, Kuang, Leal, Lee, Levine, Lu, Malla, Manjunath, Mordatch, Nachum, Parada, Peralta, Perez, Pertsch, Quiambao, Rao, Ryoo, Salazar, Sanketi, Sayed, Singh, Sontakke, Stone, Tan, Tran, Vanhoucke, Vega, Vuong, Xia, Xiao, Xu, Xu, Yu, and Zitkovich]{rt12022arxiv}
A.~Brohan, N.~Brown, J.~Carbajal, Y.~Chebotar, J.~Dabis, C.~Finn, K.~Gopalakrishnan, K.~Hausman, A.~Herzog, J.~Hsu, J.~Ibarz, B.~Ichter, A.~Irpan, T.~Jackson, S.~Jesmonth, N.~Joshi, R.~Julian, D.~Kalashnikov, Y.~Kuang, I.~Leal, K.-H. Lee, S.~Levine, Y.~Lu, U.~Malla, D.~Manjunath, I.~Mordatch, O.~Nachum, C.~Parada, J.~Peralta, E.~Perez, K.~Pertsch, J.~Quiambao, K.~Rao, M.~S. Ryoo, G.~Salazar, P.~R. Sanketi, K.~Sayed, J.~Singh, S.~Sontakke, A.~Stone, C.~Tan, H.~Tran, V.~Vanhoucke, S.~Vega, Q.~H. Vuong, F.~Xia, T.~Xiao, P.~Xu, S.~Xu, T.~Yu, and B.~Zitkovich.
\newblock {RT-1: Robotics Transformer for Real-World Control at Scale}.
\newblock In \emph{Proceedings of Robotics: Science and Systems}, Daegu, Republic of Korea, July 2023.
\newblock \doi{10.15607/RSS.2023.XIX.025}.

\bibitem[Zitkovich et~al.(2023)Zitkovich, Yu, Xu, Xu, Xiao, Xia, Wu, Wohlhart, Welker, Wahid, Vuong, Vanhoucke, Tran, Soricut, Singh, Singh, Sermanet, Sanketi, Salazar, Ryoo, Reymann, Rao, Pertsch, Mordatch, Michalewski, Lu, Levine, Lee, Lee, Leal, Kuang, Kalashnikov, Julian, Joshi, Irpan, Ichter, Hsu, Herzog, Hausman, Gopalakrishnan, Fu, Florence, Finn, Dubey, Driess, Ding, Choromanski, Chen, Chebotar, Carbajal, Brown, Brohan, Arenas, and Han]{rt22023arxiv}
B.~Zitkovich, T.~Yu, S.~Xu, P.~Xu, T.~Xiao, F.~Xia, J.~Wu, P.~Wohlhart, S.~Welker, A.~Wahid, Q.~Vuong, V.~Vanhoucke, H.~Tran, R.~Soricut, A.~Singh, J.~Singh, P.~Sermanet, P.~R. Sanketi, G.~Salazar, M.~S. Ryoo, K.~Reymann, K.~Rao, K.~Pertsch, I.~Mordatch, H.~Michalewski, Y.~Lu, S.~Levine, L.~Lee, T.-W.~E. Lee, I.~Leal, Y.~Kuang, D.~Kalashnikov, R.~Julian, N.~J. Joshi, A.~Irpan, B.~Ichter, J.~Hsu, A.~Herzog, K.~Hausman, K.~Gopalakrishnan, C.~Fu, P.~Florence, C.~Finn, K.~A. Dubey, D.~Driess, T.~Ding, K.~M. Choromanski, X.~Chen, Y.~Chebotar, J.~Carbajal, N.~Brown, A.~Brohan, M.~G. Arenas, and K.~Han.
\newblock Rt-2: Vision-language-action models transfer web knowledge to robotic control.
\newblock In J.~Tan, M.~Toussaint, and K.~Darvish, editors, \emph{Proceedings of The 7th Conference on Robot Learning}, volume 229 of \emph{Proceedings of Machine Learning Research}, pages 2165--2183. PMLR, 06--09 Nov 2023.
\newblock URL \url{https://proceedings.mlr.press/v229/zitkovich23a.html}.

\bibitem[{Octo Model Team} et~al.(2024){Octo Model Team}, Ghosh, Walke, Pertsch, Black, Mees, Dasari, Hejna, Xu, Luo, Kreiman, Tan, Chen, Sanketi, Vuong, Xiao, Sadigh, Finn, and Levine]{octo_2023}
{Octo Model Team}, D.~Ghosh, H.~Walke, K.~Pertsch, K.~Black, O.~Mees, S.~Dasari, J.~Hejna, C.~Xu, J.~Luo, T.~Kreiman, Y.~Tan, L.~Y. Chen, P.~Sanketi, Q.~Vuong, T.~Xiao, D.~Sadigh, C.~Finn, and S.~Levine.
\newblock Octo: An open-source generalist robot policy.
\newblock In \emph{Proceedings of Robotics: Science and Systems}, Delft, Netherlands, 2024.

\bibitem[Kim et~al.(2025)Kim, Pertsch, Karamcheti, Xiao, Balakrishna, Nair, Rafailov, Foster, Sanketi, Vuong, Kollar, Burchfiel, Tedrake, Sadigh, Levine, Liang, and Finn]{pmlr-v270-kim25c}
M.~J. Kim, K.~Pertsch, S.~Karamcheti, T.~Xiao, A.~Balakrishna, S.~Nair, R.~Rafailov, E.~P. Foster, P.~R. Sanketi, Q.~Vuong, T.~Kollar, B.~Burchfiel, R.~Tedrake, D.~Sadigh, S.~Levine, P.~Liang, and C.~Finn.
\newblock Openvla: An open-source vision-language-action model.
\newblock In P.~Agrawal, O.~Kroemer, and W.~Burgard, editors, \emph{Proceedings of The 8th Conference on Robot Learning}, volume 270 of \emph{Proceedings of Machine Learning Research}, pages 2679--2713. PMLR, 06--09 Nov 2025.

\bibitem[Black et~al.(2024)Black, Brown, Driess, Esmail, Equi, Finn, Fusai, Groom, Hausman, Ichter, et~al.]{black2410pi0}
K.~Black, N.~Brown, D.~Driess, A.~Esmail, M.~Equi, C.~Finn, N.~Fusai, L.~Groom, K.~Hausman, B.~Ichter, et~al.
\newblock $\pi$0: A vision-language-action flow model for general robot control.
\newblock \emph{URL https://arxiv.org/abs/2410.24164}, 2024.

\bibitem[Mandlekar et~al.(2023)Mandlekar, Nasiriany, Wen, Akinola, Narang, Fan, Zhu, and Fox]{mandlekar2023mimicgen}
A.~Mandlekar, S.~Nasiriany, B.~Wen, I.~Akinola, Y.~Narang, L.~Fan, Y.~Zhu, and D.~Fox.
\newblock Mimicgen: A data generation system for scalable robot learning using human demonstrations.
\newblock In J.~Tan, M.~Toussaint, and K.~Darvish, editors, \emph{Proceedings of The 7th Conference on Robot Learning}, volume 229 of \emph{Proceedings of Machine Learning Research}, pages 1820--1864. PMLR, 06--09 Nov 2023.

\bibitem[Yu et~al.(2023)Yu, Xiao, Tompson, Stone, Wang, Brohan, Singh, Tan, M, Peralta, Hausman, Ichter, and Xia]{yu2023scaling}
T.~Yu, T.~Xiao, J.~Tompson, A.~Stone, S.~Wang, A.~Brohan, J.~Singh, C.~Tan, D.~M, J.~Peralta, K.~Hausman, B.~Ichter, and F.~Xia.
\newblock {Scaling Robot Learning with Semantically Imagined Experience}.
\newblock In \emph{Proceedings of Robotics: Science and Systems}, Daegu, Republic of Korea, July 2023.
\newblock \doi{10.15607/RSS.2023.XIX.027}.

\bibitem[Mandi et~al.(2022)Mandi, Bharadhwaj, Moens, Song, Rajeswaran, and Kumar]{mandi2022cacti}
Z.~Mandi, H.~Bharadhwaj, V.~Moens, S.~Song, A.~Rajeswaran, and V.~Kumar.
\newblock Cacti: A framework for scalable multi-task multi-scene visual imitation learning.
\newblock \emph{arXiv preprint arXiv:2212.05711}, 2022.

\bibitem[Smith et~al.(2024)Smith, Irpan, Arenas, Kirmani, Kalashnikov, Shah, and Xiao]{smith2024steer}
L.~Smith, A.~Irpan, M.~G. Arenas, S.~Kirmani, D.~Kalashnikov, D.~Shah, and T.~Xiao.
\newblock Steer: Flexible robotic manipulation via dense language grounding.
\newblock \emph{arXiv preprint arXiv:2411.03409}, 2024.

\bibitem[Zawalski et~al.(2025)Zawalski, Chen, Pertsch, Mees, Finn, and Levine]{pmlr-v270-zawalski25a}
M.~Zawalski, W.~Chen, K.~Pertsch, O.~Mees, C.~Finn, and S.~Levine.
\newblock Robotic control via embodied chain-of-thought reasoning.
\newblock In P.~Agrawal, O.~Kroemer, and W.~Burgard, editors, \emph{Proceedings of The 8th Conference on Robot Learning}, volume 270 of \emph{Proceedings of Machine Learning Research}, pages 3157--3181. PMLR, 06--09 Nov 2025.

\bibitem[Hejna et~al.(2025)Hejna, Bhateja, Jiang, Pertsch, and Sadigh]{pmlr-v270-hejna25a}
J.~Hejna, C.~A. Bhateja, Y.~Jiang, K.~Pertsch, and D.~Sadigh.
\newblock Remix: Optimizing data mixtures for large scale imitation learning.
\newblock In P.~Agrawal, O.~Kroemer, and W.~Burgard, editors, \emph{Proceedings of The 8th Conference on Robot Learning}, volume 270 of \emph{Proceedings of Machine Learning Research}, pages 145--164. PMLR, 06--09 Nov 2025.

\bibitem[Dass et~al.(2025)Dass, Khaddaj, Engstrom, Madry, Ilyas, and Mart{\'\i}n-Mart{\'\i}n]{dass2025datamil}
S.~Dass, A.~Khaddaj, L.~Engstrom, A.~Madry, A.~Ilyas, and R.~Mart{\'\i}n-Mart{\'\i}n.
\newblock Datamil: Selecting data for robot imitation learning with datamodels.
\newblock \emph{arXiv preprint arXiv:2505.09603}, 2025.

\bibitem[Shah et~al.(2023)Shah, Park, Ilyas, and Madry]{shah2023modeldiff}
H.~Shah, S.~M. Park, A.~Ilyas, and A.~Madry.
\newblock Modeldiff: A framework for comparing learning algorithms.
\newblock In \emph{International Conference on Machine Learning}, pages 30646--30688. PMLR, 2023.

\bibitem[Ghorbani and Zou(2019)]{ghorbani2019data}
A.~Ghorbani and J.~Zou.
\newblock Data shapley: Equitable valuation of data for machine learning.
\newblock In \emph{International conference on machine learning}, pages 2242--2251. PMLR, 2019.

\bibitem[Choe et~al.(2024)Choe, Ahn, Bae, Zhao, Kang, Chung, Pratapa, Neiswanger, Strubell, Mitamura, et~al.]{choe2024your}
S.~K. Choe, H.~Ahn, J.~Bae, K.~Zhao, M.~Kang, Y.~Chung, A.~Pratapa, W.~Neiswanger, E.~Strubell, T.~Mitamura, et~al.
\newblock What is your data worth to gpt? llm-scale data valuation with influence functions.
\newblock \emph{arXiv preprint arXiv:2405.13954}, 2024.

\bibitem[Georgiev et~al.(2025)Georgiev, Rinberg, Park, Garg, Ilyas, Madry, and Neel]{georgiev2024attribute}
K.~Georgiev, R.~Rinberg, S.~M. Park, S.~Garg, A.~Ilyas, A.~Madry, and S.~Neel.
\newblock Attribute-to-delete: Machine unlearning via datamodel matching.
\newblock In \emph{The Thirteenth International Conference on Learning Representations}, 2025.
\newblock URL \url{https://openreview.net/forum?id=3vXpZpOn29}.

\bibitem[Madry et~al.(2024)Madry, Ilyas, Engstrom, Park, and Georgiev]{madry2024icml}
A.~Madry, A.~Ilyas, L.~Engstrom, S.~M. Park, and K.~Georgiev.
\newblock Data attribution at scale.
\newblock \url{https://ml-data-tutorial.org/}, 2024.
\newblock Tutorial at ICML 2024.

\bibitem[Basu et~al.(2021)Basu, Pope, and Feizi]{basu2021influence}
S.~Basu, P.~Pope, and S.~Feizi.
\newblock Influence functions in deep learning are fragile.
\newblock In \emph{International Conference on Learning Representations}, 2021.
\newblock URL \url{https://openreview.net/forum?id=xHKVVHGDOEk}.

\bibitem[Bae et~al.(2022)Bae, Ng, Lo, Ghassemi, and Grosse]{bae2022if}
J.~Bae, N.~Ng, A.~Lo, M.~Ghassemi, and R.~B. Grosse.
\newblock If influence functions are the answer, then what is the question?
\newblock \emph{Advances in Neural Information Processing Systems}, 35:\penalty0 17953--17967, 2022.

\bibitem[Ilyas and Engstrom(2025)]{ilyas2025magic}
A.~Ilyas and L.~Engstrom.
\newblock Magic: Near-optimal data attribution for deep learning.
\newblock \emph{arXiv preprint arXiv:2504.16430}, 2025.

\bibitem[Zheng et~al.(2024)Zheng, Pang, Du, Jiang, and Lin]{zheng2023intriguing}
X.~Zheng, T.~Pang, C.~Du, J.~Jiang, and M.~Lin.
\newblock Intriguing properties of data attribution on diffusion models.
\newblock In \emph{The Twelfth International Conference on Learning Representations}, 2024.
\newblock URL \url{https://openreview.net/forum?id=vKViCoKGcB}.

\bibitem[Georgiev et~al.(2023)Georgiev, Vendrow, Salman, Park, and Madry]{georgiev2023journey}
K.~Georgiev, J.~Vendrow, H.~Salman, S.~M. Park, and A.~Madry.
\newblock The journey, not the destination: How data guides diffusion models.
\newblock \emph{arXiv preprint arXiv:2312.06205}, 2023.

\bibitem[Xia et~al.(2024)Xia, Malladi, Gururangan, Arora, and Chen]{xia2024less}
M.~Xia, S.~Malladi, S.~Gururangan, S.~Arora, and D.~Chen.
\newblock Less: Selecting influential data for targeted instruction tuning.
\newblock In \emph{International Conference on Machine Learning}, pages 54104--54132. PMLR, 2024.

\bibitem[Liu et~al.(2024)Liu, Karbasi, and Rekatsinas]{liu2024tsds}
Z.~Liu, A.~Karbasi, and T.~Rekatsinas.
\newblock Tsds: Data selection for task-specific model finetuning.
\newblock \emph{Advances in Neural Information Processing Systems}, 37, 2024.

\bibitem[Engstrom et~al.(2025)Engstrom, Ilyas, Chen, Feldmann, Moses, and Madry]{engstrom2025optimizing}
L.~Engstrom, A.~Ilyas, B.~Chen, A.~Feldmann, W.~Moses, and A.~Madry.
\newblock Optimizing ml training with metagradient descent.
\newblock \emph{arXiv preprint arXiv:2503.13751}, 2025.

\bibitem[Belkhale et~al.(2023)Belkhale, Cui, and Sadigh]{belkhale2023data}
S.~Belkhale, Y.~Cui, and D.~Sadigh.
\newblock Data quality in imitation learning.
\newblock \emph{Advances in neural information processing systems}, 36:\penalty0 80375--80395, 2023.

\bibitem[Vincent et~al.(2024)Vincent, Nishimura, Itkina, Shah, Schwager, and Kollar]{vincent2024generalizable}
J.~A. Vincent, H.~Nishimura, M.~Itkina, P.~Shah, M.~Schwager, and T.~Kollar.
\newblock How generalizable is my behavior cloning policy? a statistical approach to trustworthy performance evaluation.
\newblock \emph{IEEE Robotics and Automation Letters}, 2024.

\bibitem[Gandhi et~al.(2023)Gandhi, Karamcheti, Liao, and Sadigh]{gandhi2023eliciting}
K.~Gandhi, S.~Karamcheti, M.~Liao, and D.~Sadigh.
\newblock Eliciting compatible demonstrations for multi-human imitation learning.
\newblock In K.~Liu, D.~Kulic, and J.~Ichnowski, editors, \emph{Proceedings of The 6th Conference on Robot Learning}, volume 205 of \emph{Proceedings of Machine Learning Research}, pages 1981--1991. PMLR, 14--18 Dec 2023.

\bibitem[Cui et~al.(2019)Cui, Isele, Niekum, and Fujimura]{cui2019uncertainty}
Y.~Cui, D.~Isele, S.~Niekum, and K.~Fujimura.
\newblock Uncertainty-aware data aggregation for deep imitation learning.
\newblock In \emph{2019 International Conference on Robotics and Automation (ICRA)}, pages 761--767. IEEE, 2019.

\bibitem[Ross et~al.(2011)Ross, Gordon, and Bagnell]{ross2011reduction}
S.~Ross, G.~Gordon, and D.~Bagnell.
\newblock A reduction of imitation learning and structured prediction to no-regret online learning.
\newblock In \emph{Proceedings of the fourteenth international conference on artificial intelligence and statistics}, pages 627--635. JMLR Workshop and Conference Proceedings, 2011.

\bibitem[Williams(1992)]{williams1992simple}
R.~J. Williams.
\newblock Simple statistical gradient-following algorithms for connectionist reinforcement learning.
\newblock \emph{Machine learning}, 8:\penalty0 229--256, 1992.

\bibitem[Ilyas et~al.(2022)Ilyas, Park, Engstrom, Leclerc, and Madry]{ilyas2022datamodels}
A.~Ilyas, S.~M. Park, L.~Engstrom, G.~Leclerc, and A.~Madry.
\newblock Datamodels: Understanding predictions with data and data with predictions.
\newblock In K.~Chaudhuri, S.~Jegelka, L.~Song, C.~Szepesvari, G.~Niu, and S.~Sabato, editors, \emph{Proceedings of the 39th International Conference on Machine Learning}, volume 162 of \emph{Proceedings of Machine Learning Research}, pages 9525--9587. PMLR, 17--23 Jul 2022.

\bibitem[Mandlekar et~al.(2022)Mandlekar, Xu, Wong, Nasiriany, Wang, Kulkarni, Fei-Fei, Savarese, Zhu, and Mart\'in-Mart\'in]{pmlr-v164-mandlekar22a}
A.~Mandlekar, D.~Xu, J.~Wong, S.~Nasiriany, C.~Wang, R.~Kulkarni, L.~Fei-Fei, S.~Savarese, Y.~Zhu, and R.~Mart\'in-Mart\'in.
\newblock What matters in learning from offline human demonstrations for robot manipulation.
\newblock In A.~Faust, D.~Hsu, and G.~Neumann, editors, \emph{Proceedings of the 5th Conference on Robot Learning}, volume 164 of \emph{Proceedings of Machine Learning Research}, pages 1678--1690. PMLR, 08--11 Nov 2022.

\bibitem[Chi et~al.(2023)Chi, Xu, Feng, Cousineau, Du, Burchfiel, Tedrake, and Song]{chi2023diffusion}
C.~Chi, Z.~Xu, S.~Feng, E.~Cousineau, Y.~Du, B.~Burchfiel, R.~Tedrake, and S.~Song.
\newblock Diffusion policy: Visuomotor policy learning via action diffusion.
\newblock \emph{The International Journal of Robotics Research}, page 02783649241273668, 2023.

\bibitem[Greensmith et~al.(2004)Greensmith, Bartlett, and Baxter]{greensmith2004variance}
E.~Greensmith, P.~L. Bartlett, and J.~Baxter.
\newblock Variance reduction techniques for gradient estimates in reinforcement learning.
\newblock \emph{Journal of Machine Learning Research}, 5\penalty0 (Nov):\penalty0 1471--1530, 2004.

\bibitem[Settles(2012)]{Settles2012}
B.~Settles.
\newblock \emph{Active Learning}.
\newblock Morgan \& Claypool Publishers, 2012.

\bibitem[Ho et~al.(2020)Ho, Jain, and Abbeel]{ho2020denoising}
J.~Ho, A.~Jain, and P.~Abbeel.
\newblock Denoising diffusion probabilistic models.
\newblock \emph{Advances in neural information processing systems}, 33:\penalty0 6840--6851, 2020.

\bibitem[Song et~al.(2021)Song, Sohl-Dickstein, Kingma, Kumar, Ermon, and Poole]{song2020score}
Y.~Song, J.~Sohl-Dickstein, D.~P. Kingma, A.~Kumar, S.~Ermon, and B.~Poole.
\newblock Score-based generative modeling through stochastic differential equations.
\newblock In \emph{International Conference on Learning Representations}, 2021.
\newblock URL \url{https://openreview.net/forum?id=PxTIG12RRHS}.

\bibitem[Lin et~al.(2025)Lin, Tao, Dong, and Xu]{lin2024diffusion}
J.~Lin, L.~Tao, M.~Dong, and C.~Xu.
\newblock Diffusion attribution score: Evaluating training data influence in diffusion model.
\newblock In \emph{The Thirteenth International Conference on Learning Representations}, 2025.
\newblock URL \url{https://openreview.net/forum?id=kuutidLf6R}.

\bibitem[Martens(2020)]{martens2020insights}
J.~Martens.
\newblock New insights and perspectives on the natural gradient method.
\newblock \emph{Journal of Machine Learning Research}, 21\penalty0 (146):\penalty0 1--76, 2020.
\newblock URL \url{http://jmlr.org/papers/v21/17-678.html}.

\bibitem[Mlodozeniec et~al.(2025)Mlodozeniec, Eschenhagen, Bae, Immer, Krueger, and Turner]{mlodozeniec2024influence}
B.~K. Mlodozeniec, R.~Eschenhagen, J.~Bae, A.~Immer, D.~Krueger, and R.~E. Turner.
\newblock Influence functions for scalable data attribution in diffusion models.
\newblock In \emph{The Thirteenth International Conference on Learning Representations}, 2025.
\newblock URL \url{https://openreview.net/forum?id=esYrEndGsr}.

\bibitem[Xie et~al.(2024)Xie, Li, Bai, and Hsieh]{xie2024data}
T.~Xie, H.~Li, A.~Bai, and C.-J. Hsieh.
\newblock Data attribution for diffusion models: Timestep-induced bias in influence estimation.
\newblock \emph{Transactions on Machine Learning Research}, 2024.
\newblock ISSN 2835-8856.
\newblock URL \url{https://openreview.net/forum?id=P3Lyun7CZs}.

\bibitem[Johnson et~al.(1984)Johnson, Lindenstrauss, et~al.]{johnson1984extensions}
W.~B. Johnson, J.~Lindenstrauss, et~al.
\newblock Extensions of lipschitz mappings into a hilbert space.
\newblock \emph{Contemporary mathematics}, 26\penalty0 (189-206):\penalty0 1, 1984.

\bibitem[Khatib(2003)]{khatib2003unified}
O.~Khatib.
\newblock A unified approach for motion and force control of robot manipulators: The operational space formulation.
\newblock \emph{IEEE Journal on Robotics and Automation}, 3\penalty0 (1):\penalty0 43--53, 2003.

\bibitem[Hu et~al.(2022)Hu, Shen, Wallis, Allen-Zhu, Li, Wang, Wang, Chen, et~al.]{hu2022lora}
E.~J. Hu, Y.~Shen, P.~Wallis, Z.~Allen-Zhu, Y.~Li, S.~Wang, L.~Wang, W.~Chen, et~al.
\newblock Lora: Low-rank adaptation of large language models.
\newblock \emph{ICLR}, 1\penalty0 (2):\penalty0 3, 2022.

\bibitem[Agia et~al.(2025)Agia, Sinha, Yang, Cao, Antonova, Pavone, and Bohg]{agia2024unpacking}
C.~Agia, R.~Sinha, J.~Yang, Z.~Cao, R.~Antonova, M.~Pavone, and J.~Bohg.
\newblock Unpacking failure modes of generative policies: Runtime monitoring of consistency and progress.
\newblock In P.~Agrawal, O.~Kroemer, and W.~Burgard, editors, \emph{Proceedings of The 8th Conference on Robot Learning}, volume 270 of \emph{Proceedings of Machine Learning Research}, pages 689--723. PMLR, 06--09 Nov 2025.

\bibitem[Dai et~al.(2024)Dai, Lee, Fazeli, and Chai]{dai2024racer}
Y.~Dai, J.~Lee, N.~Fazeli, and J.~Chai.
\newblock Racer: Rich language-guided failure recovery policies for imitation learning.
\newblock \emph{arXiv preprint arXiv:2409.14674}, 2024.

\bibitem[Sinha et~al.(2024)Sinha, Elhafsi, Agia, Foutter, Schmerling, and Pavone]{sinha2024real}
R.~Sinha, A.~Elhafsi, C.~Agia, M.~Foutter, E.~Schmerling, and M.~Pavone.
\newblock {Real-Time Anomaly Detection and Reactive Planning with Large Language Models}.
\newblock In \emph{Proceedings of Robotics: Science and Systems}, Delft, Netherlands, July 2024.
\newblock \doi{10.15607/RSS.2024.XX.114}.

\bibitem[He et~al.(2016)He, Zhang, Ren, and Sun]{he2016deep}
K.~He, X.~Zhang, S.~Ren, and J.~Sun.
\newblock Deep residual learning for image recognition.
\newblock In \emph{Proceedings of the IEEE conference on computer vision and pattern recognition}, pages 770--778, 2016.

\bibitem[Oquab et~al.(2023)Oquab, Darcet, Moutakanni, Vo, Szafraniec, Khalidov, Fernandez, Haziza, Massa, El-Nouby, et~al.]{oquab2023dinov2}
M.~Oquab, T.~Darcet, T.~Moutakanni, H.~Vo, M.~Szafraniec, V.~Khalidov, P.~Fernandez, D.~Haziza, F.~Massa, A.~El-Nouby, et~al.
\newblock Dinov2: Learning robust visual features without supervision.
\newblock \emph{arXiv preprint arXiv:2304.07193}, 2023.

\end{thebibliography}

\newpage
\newlength{\defaultparskip}
\setlength{\defaultparskip}{\parskip}
\setlength{\parskip}{1em}

\section*{Appendix Overview -- Curating Data your Robot Loves with Influence Functions}
The appendix offers additional details \textit{w.r.t.} the implementation of \basemethod{} (\cref{appx:method}), the experiments conducted (\cref{appx:experiments}), along with extended results and analysis (\cref{appx:results}), and finally, supporting derivations for our data curation methods (\cref{appx:derivations}). Videos and code are made available at: \href{https://cupid-curation.github.io}{https://cupid-curation.github.io}.

\begin{appendices}

\startcontents[sections]
\printcontents[sections]{l}{1}{\setcounter{tocdepth}{3}}

\clearpage

\setlength{\parskip}{\defaultparskip}

\section{Implementation Details}\label{appx:method}

\subsection{Influence Functions for Diffusion Policies}\label{appx:actinf-dp}

For ease of reference in this section, we restate the definition of the action influence (\cref{def:actinf}) and the proposition establishing performance influence (\cref{prop:polinf}), both originally introduced in \cref{sec:method}.

\paragraph{Restatement of \cref{def:actinf}.}\hspace{-0.8em} 
\textit{The \emph{\textbf{action influence}} of a state-action pair $(s,a)$ on a test state-action pair $(s',a')$ is the influence of $(s,a)$ on the policy's log-likelihood $ \log \pi_\theta(a'|s')$. That is,}
\begin{equation*}
    \actinf((s',a'), (s,a)) := -\nabla_\theta \log\pi_\theta(a'|s')^\top H_{\mathrm{bc}}^{-1}\nabla_\theta \ell(s, a; \pi_\theta).
\end{equation*}

\paragraph{Restatement of \cref{prop:polinf}.}\hspace{-0.8em}
\textit{Assume that $\theta(\calD) = \arg\min_{\theta'} \calL_{\mathrm{bc}}(\theta'; \calD)$, that $\calL_{\mathrm{bc}}$ is twice differentiable in $\theta$, and that $H_{\mathrm{bc}} \succ 0$ is positive definite (i.e., $\theta(\calD)$ is not a saddle point)\footnoteref{fn:track}. Then, it holds that}
\begin{equation*}
    \polinf(\xi) = \E_{\tau\sim p(\tau|\pi_\theta)}\bigg[\frac{R(\tau)}{H} \sum_{(s',a')\in\tau}\sum_{(s,a)\in\xi}\actinf\big((s',a'),(s,a)\big)\bigg].
\end{equation*}
\textit{where $\polinf(\xi)$ is the \emph{\textbf{performance influence}} of a demonstration $\xi$ (as introduced in \cref{def:polinf}).}

\subsubsection*{Computing the Action Influence}
Although \cref{prop:polinf} provides a clean mechanism to attribute policy performance to its training data by leveraging influence scores on action log-likelihoods, computing $\nabla_\theta \log\pi_\theta(a'|s')$ (in the action influence $\actinf$) for diffusion-based policy architectures is nontrivial due to the iterative denoising process~\cite{ho2020denoising, song2020score}. 
Instead, various works outside robotics propose to approximate the log-likelihood with the denoising loss $\ell(s', a'; \pi_{\theta})$ for the purpose of data attribution~\cite{georgiev2023journey}, because the denoising loss is proportionate to the variational lower bound on $\log \pi_\theta(a'|s')$. In \cref{sec:experiments}, we apply a similar approximation to perform data attribution on state-of-the-art diffusion policies~\cite{chi2023diffusion}, which we describe below.

\textbf{Diffusion Policy:} Consider the standard diffusion policy architecture~\cite{chi2023diffusion}. An action $a := a^0$ is generated by iteratively denoising an initially random action $a^T \sim \calN(0, 1)$ over $T$ steps as $a^T, \ldots, a^0$ using a noise prediction network $\epsilon_\theta$, where $a^i$ denotes the generated action at the $i$-th denoising iteration. Following the imitation learning setting described in \cref{sec:formulation}, the parameters $\theta$ of the noise prediction network $\epsilon_\theta$ are fit to the BC objective as $\theta = \arg\min_{\theta'} \{\mathcal{L}_{\text{bc}}(\theta'; \mathcal{D}) := \frac{1}{|\calD| H}\sum_{\xi^i\in\calD}\sum_{(s, a) \in \xi^i} \ell(s, a; \pi_{\theta'})\}$.
Here, the noise prediction network $\epsilon_\theta$ is trained to predict random noise $\epsilon^i \sim \calN(0, 1)$ added to the action $a$ at randomly sampled timesteps $i\sim \calU[0, T)$ of the diffusion process using the loss function $\ell$ defined as
\begin{equation}\label{eq:dp-loss}
    \ell(s, a; \pi_{\theta'}) := \mathbb{E}_{\epsilon^i, i} \left[||\epsilon^i - \epsilon_{\theta'}(\sqrt{\bar\alpha_i} a + \sqrt{1 - \bar{\alpha}_i}\epsilon^i, s, i)||^2 \right],
\end{equation}
where the constants $\bar \alpha_i$ depend on the chosen noise schedule of the diffusion process. 

\textbf{Influence Approximations:} Since the denoising loss $\ell$ in \cref{eq:dp-loss} is proportionate to the variational lower bound on the action log-likelihood $\log \pi_\theta(a|s)$, it may seem intuitive to substitute $\nabla_\theta\log \pi_\theta(a'|s')$ with $-\nabla_\theta\ell(s', a'; \pi_\theta)$---assuming gradient alignment---to approximate the action influence (\cref{eq:actinf}) as
\begin{equation}\label{eq:actinf-approx-loss}
    \actinf((s',a'), (s,a)) \approx \nabla_\theta\ell(s', a'; \pi_\theta)^\top H_{\mathrm{bc}}^{-1}\nabla_\theta \ell(s, a; \pi_\theta).
\end{equation}
A similar approach is taken by \citet{georgiev2023journey} for attributing the generations of image-based diffusion models. However, consistent with more recent results in the data attribution literature~\cite{zheng2023intriguing, lin2024diffusion}, we find this approximation to work poorly in practice, with highly influential training samples $(s, a) \in \mathcal{D}$ rarely reflecting the test-time transitions $(s', a') \in \tau$ over which the action influences are computed. 
Instead, we follow the approach of \citet{zheng2023intriguing}, which entails replacing both $\log \pi_\theta(a'|s')$ and $\ell(s, a; \pi_\theta)$ in \cref{eq:actinf} with a surrogate, label-agnostic output function $\ell_{\mathrm{square}}(s, a; \pi_\theta) := \mathbb{E}_{\epsilon^i, i} [||\epsilon_{\theta}(\sqrt{\bar\alpha_i} a + \sqrt{1 - \bar{\alpha}_i}\epsilon^i, s, i)||^2]$, making our final approximation of the action influence
\begin{equation}\label{eq:actinf-approx-square}
    \actinf((s',a'), (s,a)) \approx \nabla_\theta \ell_{\mathrm{square}}(s', a'; \pi_\theta)^\top  H_{\mathrm{square}}^{-1}\nabla_\theta \ell_{\mathrm{square}}(s, a; \pi_\theta).
\end{equation}
Here, $H_{\mathrm{square}} = \frac{1}{|\calD| H}\sum_{\xi^i\in\calD}\sum_{(s, a) \in \xi^i} \nabla_\theta \ell_{\mathrm{square}}(s, a; \pi_\theta) \nabla_\theta \ell_{\mathrm{square}}(s, a; \pi_\theta)^\top$ is the Gauss-Newton approximation of the Hessian---as introduced by \citet{martens2020insights} and applied for stable and efficient influence estimation in \cite{park2023trak, bae2022if}---under the surrogate output function $\ell_{\mathrm{square}}$. 

\textbf{Additional Remarks:} While the use of $\ell_{\mathrm{square}}$ may seem counterintuitive at first, it offers three key advantages for computing action influences: 
\begin{enumerate}
    \item Leave-one-out influences (\cref{sec:background}) computed using $\ell_{\mathrm{square}}$ (\cref{eq:actinf-approx-square}) are empirically found to correlate better with actual changes in a diffusion model's loss---i.e., the difference $\ell(s', a'; \pi_{\theta(\mathcal{D} \setminus (s, a))}) - \ell(s', a'; \pi_{\theta(\mathcal{D})})$---than those computed using the loss $\ell$ (\cref{eq:actinf-approx-loss})~\cite{zheng2023intriguing}.
    \item Theoretical analysis also shows that $\ell_{\mathrm{square}}$ more closely aligns with a distributional formulation of the leave-one-out influence compared to the loss $\ell$~\cite{lin2024diffusion}. In the case of diffusion policies, this distributional formulation would seek to design $\actinf$ such that it approximates the \textit{leave-one-out divergence} $\actinf((s', a')), (s, a)) \approx D_{\mathrm{KL}}( \pi_{\theta(\mathcal{D})} (a' | s') || \pi_{\theta(\mathcal{D} \setminus (s, a))}(a' | s'))$.
    \item Using $\ell_{\mathrm{square}}$ significantly reduces the computational cost of computing action influences for policies with high-dimensional action spaces, because the $\ell^2$-norm collapses the model's prediction into a scalar $||\epsilon_{\theta}(\sqrt{\bar\alpha_i} a + \sqrt{1 - \bar{\alpha}_i}\epsilon^i, s, i)||^2$. As a result, computing \cref{eq:actinf-approx-square} requires only a single model gradient $\nabla_\theta \ell_{\mathrm{square}}$ per training and test sample. In contrast, while the technique proposed by \citet{lin2024diffusion} offers a more accurate estimate of the leave-one-out divergence $D_{\mathrm{KL}}( \pi_{\theta(\mathcal{D})} (a' | s') || \pi_{\theta(\mathcal{D} \setminus (s, a))}(a' | s'))$, its computational cost scales linearly with the dimensionality of the model's output, which may be prohibitive.
\end{enumerate}

\textbf{Accuracy-Efficiency Tradeoff:} We note that our approach for computing the performance influence of a demonstration (\cref{eq:policy-inf-deriv}) is agnostic to the choice of influence estimation technique~\cite{georgiev2023journey, zheng2023intriguing, lin2024diffusion, mlodozeniec2024influence, xie2024data}, allowing practitioners to trade off between accuracy and efficiency based on available computational resources, and enabling integration of improved data attribution methods (e.g., \cite{ilyas2025magic}) in the future.

\subsection{CUPID Hyperparameters}\label{appx:hyperparameters}
We use the same set of hyperparameters for \basemethod{} and \qualitymethod{} across all experiments.

\textbf{Performance Influence (\cref{eq:policy-inf-deriv}):} For all tasks, we define the trajectory return to be $R(\tau) = 1$ if $\tau$ completes the task and $R(\tau) = -1$ otherwise. As a result, every rollout trajectory $\tau \sim p(\cdot|\pi_\theta)$ provides information on the utility of each demonstration toward the policy's closed-loop performance. We also found \basemethod{} to work with alternative return definitions---for example, focusing solely on successful rollouts by setting $R(\tau) = 0$ when $\tau$ fails. However, such choices may increase sample complexity.

\textbf{Action Influence (\cref{eq:actinf-approx-square}):} The action influence requires computing the gradient of an expectation $\nabla_\theta\ell_{\mathrm{square}}(s, a; \pi_\theta) = \nabla_\theta\mathbb{E}_{\epsilon^i, i} [||\epsilon_{\theta}(\sqrt{\bar\alpha_i} a + \sqrt{1 - \bar{\alpha}_i}\epsilon^i, s, i)||^2]$. For all tasks, we approximate the expectation using a batch of $B = 64$ samples $(\epsilon^{(b)}, i^{(b)})$, where $\epsilon^{(b)} \sim \mathcal{N}(0, 1)$ and $i^{(b)} \sim \mathcal{U}[0, T)$ are sampled independently.

\textbf{Data Attribution:} We leverage TRAK~\cite{park2023trak} to efficiently compute action influences as defined in \cref{eq:actinf-approx-square}. First, TRAK uses random projections $\mathbf{P} \sim \calN(0, 1)^{p \times d}$, where $p$ is the number of model parameters and $d << p$ is the specified projection dimension, to reduce the dimensionality of the gradients as $g_\theta = \mathbf{P}^\top\nabla_\theta \ell_{\mathrm{square}}$ while preserving their inner products $g_\theta \cdot g_\theta \approx \nabla_\theta \ell_{\mathrm{square}} \cdot \nabla_\theta \ell_{\mathrm{square}}$~\cite{johnson1984extensions}. Second, TRAK ensembles influence scores over $C$ independently trained models (i.e., from different seeds) to account for non-determinism in learning. In our experiments, we use the standard projection dimension $d = 4000$ and minimize computational cost by using only a single policy checkpoint $C = 1$, noting that ensembling over $C > 1$ policy checkpoints is likely to improve the accuracy of our influence scores.

\subsection{Combining Score Functions}\label{appx:comb-score-fns}
For ease of exposition in \cref{sec:methods-quality}, we express the overall score of a demonstration as the convex combination of its performance influence and its quality score $\alpha \polinf + (1-\alpha)\Psi_{\mathrm{qual}}$, where $\alpha = 1$ and $\alpha \in [0, 1)$ instantiates \basemethod{} and \qualitymethod{}, respectively. Here, we additionally note that taking weighted combinations of score functions requires first normalizing them to equivalent scales. 
Hence, our implementation uniformly normalizes demonstration scores within the range $[0, 1]$ (i.e., producing an absolute ranking of demonstrations) for each score function $\polinf$ and $\Psi_{\mathrm{qual}}$ before combining them.
This simple approach can be applied to combine an arbitrary number of demonstration score functions.

\section{Experimental Setup}\label{appx:experiments}

\subsection{Hardware Setup}
As depicted in \cref{fig:franka-dp-results}, our hardware experiments involve a Franka FR3 manipulator robot. 
We use a single ZED 2 camera to capture RGB-D observations and disregard the depth information. 
Our image-based policies process $256\times256$ downsampled RGB observations and predict sequences of end-effector poses for the manipulator, which are tracked using operational space control~\cite{khatib2003unified}.

\subsection{Policy Architectures}

\textbf{Diffusion Policy (DP):} We use the original diffusion policy implementation\footnote{DP's open-source implementation: \texttt{\url{https://github.com/real-stanford/diffusion_policy}}.} from \citet{chi2023diffusion}. Specifically, we use the convolutional-based diffusion policy architecture for efficiency. For state-based tasks (e.g., in RoboMimic; \cref{fig:robomimic-dp-results}), actions are generated solely using the noise prediction network $\epsilon_\theta$ as described in \cref{appx:actinf-dp}. However, for image-based tasks (e.g., on hardware; \cref{fig:franka-dp-results}), the policy $\pi_\theta$ contains two sets of parameters $\theta = (\theta_o, \theta_a)$ corresponding to a ResNet-18 encoder $E_{\theta_o}$ and the noise prediction network $\epsilon_{\theta_a}$. When scoring demonstrations, we compute action influences (\cref{eq:actinf-approx-square}) over all available policy parameters $\theta$, noting that one might also consider using a subset of the parameters, e.g., those of the noise prediction network or an alternative action head, under reduced computational budgets. 

\textit{Other optimizations:} In preliminary experiments, we found that the original diffusion policy (a) was heavily over-parameterized and (b) converged in performance much earlier in training than the specified maximum number of epochs. Thus, to accelerate experimentation in RoboMimic (\cref{fig:robomimic-dp-results}), we (a) manually determined the smallest model size that performed similarly to the original policy and (b) adjusted the maximum number of epochs to the point where additional training would result in no further performance gains. Importantly, we keep the model size and training epochs consistent across all curation methods for a given RoboMimic task. For real-world hardware experiments, we use the same model size and limit the number of training steps to 200K across all tasks, similar to \citet{hejna2025robotdatacurationmutual}. All other  diffusion policy hyperparameters are consistent with the original implementation~\cite{chi2023diffusion}. 

\begin{wraptable}{r}{0.34\linewidth}
    \vspace{-1em}
    \small
    \caption{\small
        \textbf{Hyperparameter configuration} used for $\pi_0$~\cite{black2410pi0} post-training.
    }
    \label{tab:pi0-params}
    \adjustbox{max width=\linewidth}{
    \begin{tabular}{l|l}
        \toprule
        \textbf{Hyperparameter}        & \textbf{Value}                    \\
        \midrule
        Training steps                 & 30{,}000                          \\
        Batch size                     & 16                                \\
        Optimizer                      & AdamW                             \\
        Learning rate schedule         & Cosine decay                      \\
        EMA                            & Disabled                          \\
        Action chunk length            & 50 steps                          \\
        Control frequency              & 10 Hz                             \\
        Image resolution               & $224 \times 224$                  \\
        Observation history            & 1 frame                           \\
        \midrule
        VLM backbone LoRA              & Rank = 16, $\alpha = 16$          \\
        Action expert LoRA             & Rank = 32, $\alpha = 32$          \\
        \bottomrule
    \end{tabular}}
    \vspace{-1em}
\end{wraptable}

\textbf{Generalist Robot Policy ($\pi_0$):} 
We fine-tune \texttt{Physical Intelligence}'s $\pi_0$ Vision-Language-Action (VLA) policy\footnote{$\pi_0$'s open-source implementation: \texttt{\url{https://github.com/Physical-Intelligence/openpi}}.} via Low-Rank Adaptation (LoRA)~\cite{hu2022lora} on the ``Figure-8'' and ``TuckBox'' tasks. To ensure the post-trained policy's performance is solely a result of the properties of the curated dataset used for training, we use the standard fine-tuning parameter configuration from \citet{black2410pi0} and keep all hyperparameters fixed across experiments (see \cref{tab:pi0-params}). We trained on 2 NVIDIA RTX 4090 GPUs, which took approximately 15 hours under the configuration in \cref{tab:pi0-params}. In initial experiments, we found that training for 30K steps was necessary to compensate for mismatch between our robot's action space (target end-effector poses tracked via operational space control) and the action spaces used to pre-train the base $\pi_0$ policy (absolute joint angles). In addition, we found that using a descriptive prompt for the task was necessary to yield performant policies.
We kept these prompts fixed across training, evaluation, and all curation settings. For the ``TuckBox'' task, we used the instruction ``Move the blue box underneath the white shelf'' to avoid biasing the policy towards a particular behavior mode (e.g., ``sliding'' or ``pick-and-place''). For the ``Figure-8'' task, we used the instruction ``Pick up the red rope, then tie a figure 8,'' where we found the two-step instruction to increase performance over shorter instructions like ``Tie the cleat.'' Similar to the diffusion policy experiment, we fine-tune a separate $\pi_0$ model for each curation task---filter-$k$ (\cref{task:filter-k}) and select-$k$ (\cref{task:select-k})---using their corresponding base demonstration datasets.
We then fine-tune additional $\pi_0$ models on datasets curated by our methods.

\subsection{Tasks \& Datasets}\label{appx:tasks}
Here, we provide additional details regarding our real-world hardware tasks and their corresponding datasets. We refer to \citet{pmlr-v164-mandlekar22a} for details on the simulated RoboMimic benchmark. 

\textbf{Figure-8:}
A brief description of the task is provided in \cref{sec:exp-quality}. The ``Figure-8'' dataset contains 160 demonstrations evenly split across four \textit{quality tiers}. Higher quality demonstrations complete the task at a constant rate without errors, while lower-quality demonstrations vary in progression rate~\cite{agia2024unpacking} and include retry or recovery behaviors. Therefore, the ``Figure-8'' task intends to reflect a practical setting where demonstrations of varying properties are introduced during data collection, whether organically or deliberately, e.g., to improve policy robustness to recoverable failures~\cite{dai2024racer}. Therefore, we expect curation algorithms that distinguish demonstrations upon notions of quality (e.g., predictability~\cite{hejna2025robotdatacurationmutual}) to perform well on this task, which is consistent with our findings in \cref{fig:franka-dp-results}(a) and \cref{fig:franka-pi0-transfer-results}(a).

\textbf{TuckBox:}
A brief description of the task is provided in \cref{sec:exp-strategies}. As mentioned, the ``TuckBox'' dataset contains 120 demonstrations split 2:1 between two subsets: 80 demonstrations solve the task by sliding the box under the receptacle, while 40 demonstrations first reposition the box in front of the receptacle via pick-and-place. Although the sliding strategy appears more smooth and involves just a single step, it is rendered unreliable by imperceptible test-time distribution shifts to the box's mass distribution. As such, ``TuckBox'' stands conceptually opposite to ``Figure-8,'' whereby attending to heuristic properties of demonstrations (e.g., quality) may result in poor curation performance (as shown in \cref{fig:franka-dp-results}(b)).

\textbf{Bookshelf:}
A brief description of the task is provided in \cref{sec:exp-correlations}. To summarize, the robot must extract a target book that is either shelved alone---affording a simple, horizontal pulling motion---or with another book stacked on top of it (i.e., a \textit{bookstack}). In the bookstack case, the robot must extract the target book using a vertical pulling motion, such that the stacked book does not fall off the shelf in the process (see \cref{fig:franka-dp-results}(c)). In total, the ``Bookshelf'' dataset contains 120 demonstrations split across three subsets: (a) 60 demonstrations feature the target book shelved alone with a white background, (b) 20 demonstrations feature the bookstack with a white background, and (c) 40 demonstrations feature the bookstack with a dark background. All subsets feature task-irrelevant distractor books on other shelves.

\textit{Spurious correlations in training data:} Although the vertical pulling solution to the bookstack case is demonstrated in scenes with both white and dark backgrounds, the disproporionate number of demonstrations in subset (a) versus subset (b) spuriously correlates the horizontal pulling motion with the white background. Such spurious correlations may result in \textit{causal confusion}~\cite{de2019causal}, where the policy ignores the bookstack, attends the white background, and executes the failing horizontal strategy.

\textit{Spurious correlations in rollout data:} Like ``TuckBox,'' ``Bookshelf'' represents another limiting case for curating data with quality metrics~\cite{hejna2025robotdatacurationmutual}. However, it also presents an additional challenge for methods that seek to curate data using online experience~\cite{chen2025curating}. For example, approaches that attend to differences in states between successful and failed policy rollouts may be susceptible to spurious correlations in the rollout data. Consider the simple case: if we were to observe successful rollouts when the target book is shelved alone and failed rollouts when another book is stacked above the target, then training a classifier (i.e., as in Demo-SCORE~\cite{chen2025curating}) to distinguish successful from failed states may wrongly attribute failures to the presence of the stacked book. Curating demonstrations with such a classifier would, in turn, worsen the spurious correlation in the training data. Thus, we posit that handling more challenging cases of spurious correlations in real-world data will require methods that \textit{causally attribute} the outcomes of observed test-time experiences to the training data, such as \basemethod{}.

\subsection{Baseline Details}\label{appx:baselines}

\textbf{DemInf:} We use the official implementation\footnote{DemInf open-source implementation: \texttt{\url{https://github.com/jhejna/demonstration-information}}.} provided by \citet{hejna2025robotdatacurationmutual}. We note that DemInf curates data offline---that is, without using any policy rollouts---and is at present only applicable to the demonstration filtering setting (i.e., filter-$k$, as defined in \cref{task:filter-k}). 

\textbf{Demo-SCORE:} We construct our own implementation based on the description provided by the authors~\cite{chen2025curating}. Given our assumed fixed budget of $m = 100$ rollouts for RoboMimic experiments (\cref{sec:experiments}), we collect 25 rollouts from $C = 4$ policy checkpoints throughout training. We train three-layer MLP classifiers with hidden dimensions $[16, 16, 16]$ on the first three rollout sets, and select the best classifier via cross-validation on the last 25 rollouts, as described in \cite{chen2025curating}. Since we reduce the rollout budget to $m = 25$ rollouts for hardware experiments (\cref{sec:experiments}), we collect 25 rollouts from the last $C = 1$ policy checkpoint. We then train a single ResNet-18 encoder and three-layer classification head with hidden dimensions $[32, 32, 32]$ on 20 of the rollouts, leaving 5 validation rollouts to monitor for overfitting. We train all classifiers with a heavy dropout of $0.3$ and an AdamW weight decay of $0.1$ to prevent overfitting, in alignment with \cite{chen2025curating}.
Although \citet{chen2025curating} only test Demo-SCORE for demonstration filtering, we extend its use for demonstration selection (i.e., select-$k$, as defined in \cref{task:select-k}).

\textbf{Success Similarity:}
We design a custom robot data curation algorithm  that, similar to Demo-SCORE, valuates demonstrations based on a heuristic measure of similarity \textit{w.r.t.} successful policy rollouts. Instead of training classifiers, Success Similarity measures the average state-embedding similarity of a demonstration \textit{w.r.t.} all successful rollouts as 
\begin{equation*}
    S(\xi; \calD_\tau) = -\sum_{\tau \in \calD_\tau} \bigg[\mathbf{1}(R(\tau) = 1) \cdot \frac{1}{H^2} \sum_{s'\in\tau}\sum_{s\in\xi} D\big(\phi(s'), \phi(s)\big) \bigg],
\end{equation*}
where the indicator function $\mathbf{1}$ evaluates to 1 if rollout $\tau$ is successful and 0 otherwise, $H$ is the assumed length of all demonstrations $\xi \in \calD$ and rollouts $\tau \in \calD_\tau$ for notational simplicity, $\phi$ is the state embedding function, and $D$ is a specified distance function over state embeddings~\cite{sinha2024real}, such as the Mahalanobis, L2, or cosine distance. 
For image-based states, we experimented with various embedding functions $\phi$, including ResNet~\cite{he2016deep}, DINOv2~\cite{oquab2023dinov2}, and the policy's vision encoder~\cite{agia2024unpacking}, and ultimately found the policy's vision encoder to work best in RoboMimic. The embedding function is set to identity for low-dimensional states (i.e., $\phi(s) = s$). 
Lastly, the distance function $D$ is chosen for compatibility with $\phi$: e.g., L2 distance for policy encoder embeddings and cosine distance for DINOv2 embeddings. 

\textit{Comparison to Performance Influence (\basemethod{}):} One can interpret Success Similarity as replacing the action influence $\actinf((s',a'),(s,a))$ (\cref{eq:actinf}) with a state-based proxy $-D(\phi(s'), \phi(s))$ in an attempt to estimate the performance contribution of a demonstration (\cref{eq:policy-inf-deriv}). In our RoboMimic experiments (\cref{fig:robomimic-dp-results}), this approach performs comparably to Demo-SCORE and, in some cases, even outperforms it---without requiring the training of any additional models. However, Success Similarity performs consistently worse than \basemethod{} across all tasks, supporting prior findings that influence functions offer a substantially stronger causal signal than heuristic measures of similarity~\cite{park2023trak}.

\textbf{Oracle:} For each task, the Oracle method represents a best attempt to curate data assuming privileged access to ground-truth demonstration labels. For the RoboMimic and ``Figure-8'' tasks, the Oracle ranks demonstrations in descending order of quality, choosing high-quality demonstrations before low-quality demonstrations. For the ``TuckBox'' task, the Oracle first chooses all demonstrations exhibiting the more robust pick-and-place strategy before any demonstration exhibiting the more brittle sliding strategy. Lastly, for the ``Bookshelf'' task, the Oracle chooses demonstrations to minimize the effect of the \textit{known} spurious correlation (i.e., horizontal pulling motion in the presence of a white background), resulting in a more balanced curated dataset. These definitions of the Oracle apply identically to the filter-$k$ (\cref{task:filter-k}) and select-$k$ (\cref{task:select-k}) curation tasks studied throughout this work.

\textbf{Additional baselines:} We implement a number of additional custom baselines that one might try in practice, such as curating data based on policy loss, policy uncertainty, state diversity, and action diversity. However, we exclude them from our experiments given their relatively poor performance.

\section{Additional Results \& Analysis}\label{appx:results}

We present additional results and ablations for our RoboMimic and Franka real-world tasks that were cut from the main text due to space constraints. 

\subsection{Extended Discussion on RoboMimic Results (\cref{sec:discussion-properties})}\label{appx:results-robomimic-discussion}

We provide an extended discussion on \cref{sec:discussion-properties} for our RoboMimic simulation results. 

\textit{Performance versus Data Quality:} One of our key findings is that the performance of a state-of-the-art policy does not strictly correlate with the \textit{perceived quality} of its training data. Factors such as redundancy, balance, and coverage of the dataset all play a role in determining policy performance. This is illustrated in the Oracle filter-$k$ results (left three plots of \cref{fig:robomimic-dp-results}). While the top row shows a monotonic increase in average dataset quality as lower-quality demonstrations are filtered out, the bottom row reveals (1) a consistent performance drop for diffusion policies on 2 out of 3 tasks, and (2) as expected, performance degradation when too many demonstrations are removed. Similar analysis applies to the select-$k$ setting. These results highlight two important points: First, the impact of dataset curation should not be judged by quality labels alone, but by the downstream performance of models trained on curated datasets. Second, determining how much data to curate (i.e., the $k$ in filter-$k$ and select-$k$) remains another key challenge for effective data curation in practice.

\textit{Performance versus Task Complexity:} We evaluate curation performance across three RoboMimic tasks of increasing complexity---``Lift MH,'' ``Square MH,'' and ``Transport MH.'' On the simplest task, “Lift MH,” diffusion policies achieve 100\% success despite training on all demonstrations, indicating that low-quality demonstrations have minimal impact and can be safely filtered. We observe a similar trend for the moderately difficult ``Square MH'' task, where the policy benefits from access to all demonstrations regardless of their quality. However, performance degrades more quickly as demonstrations are filtered, suggesting increased sensitivity to data quantity due to the task’s higher complexity relative to ``Lift MH.'' Finally, on the challenging ``Transport MH'' task, which requires precise bi-manual coordination, both \basemethod{} and \qualitymethod{} significantly outperform the base policy. These results suggest that curation of mixed-quality datasets is most beneficial for complex, precision-critical tasks, where training on lower-quality data is more likely to hinder performance.

\subsection{Ablation on Number of Policy Rollouts in RoboMimic (\cref{sec:discussion-rollouts})}

We conduct an ablation study in RoboMimic evaluating the quality of datasets curated by \basemethod{} and \qualitymethod{} under varying numbers of rollouts, $m \in \{1, 5, 10, 25, 50, 100\}$. The results for state-based and image-based diffusion policies are shown in \cref{fig:robomimic-state-data-quality-rollout-appx} and \cref{fig:robomimic-image-data-quality-rollout-appx}, respectively. For ``Lift MH'' and ``Square MH,'' performance influences (\cref{eq:policy-inf-deriv}) stabilize around $m \in [25, 50]$, yielding quality trends similar to those obtained with $m = 100$. For ``Transport MH,'' quality trends continue to evolve until approximately $m \in [50, 100]$ rollouts, indicating that more rollouts are beneficial for accurate influence estimation in complex task settings---where curation has the greatest effect on performance.

\begin{figure}[H]
    \centering
    \includegraphics[width=\linewidth]{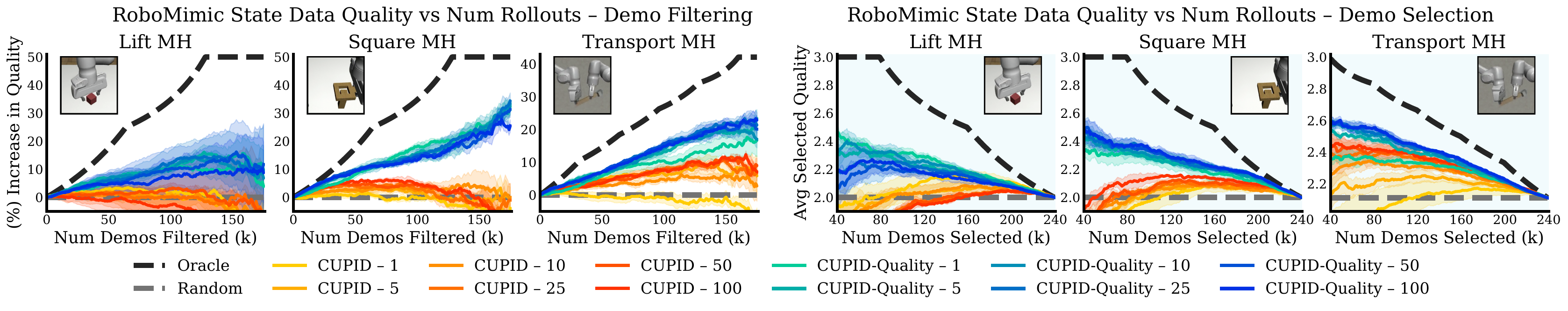}
    \vspace{-20pt}
    \caption{\small
        RoboMimic state ablation: Data quality trends under varying number of rollouts. Performance influences (\cref{eq:policy-inf-deriv}) converge around $m \in [25, 50]$ rollouts for ``Lift MH'' and ``Square MH'' (yielding similar quality trends), but continue to evolve until $m \in [50, 100]$ rollouts for ``Transport MH.'' 
        Curation performed on state-based diffusion policies.
        Results are averaged over 3 random seeds.
        Errors bars represent the standard error.
    }
    \label{fig:robomimic-state-data-quality-rollout-appx}
\end{figure}

\begin{figure}[H]
    \centering
    \includegraphics[width=\linewidth]{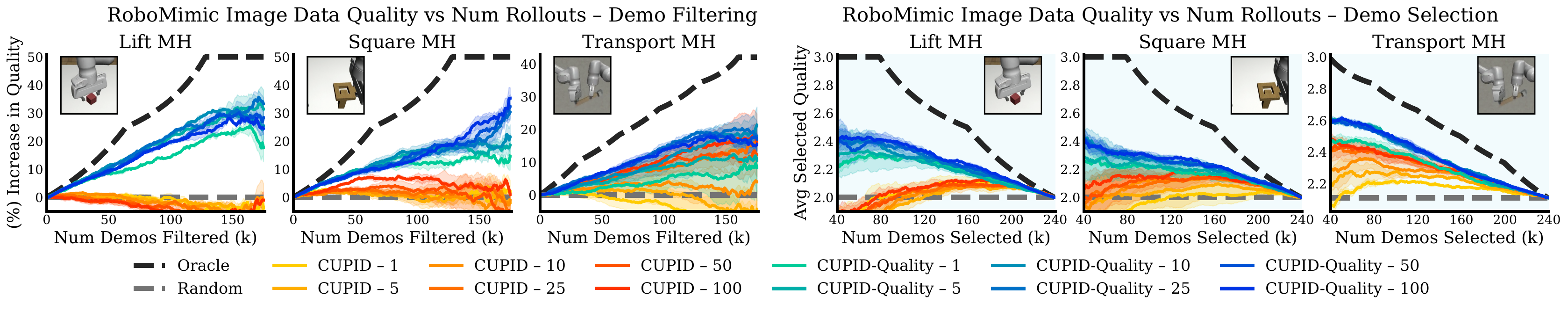}
    \vspace{-20pt}
    \caption{\small
        RoboMimic image ablation: Data quality trends under varying number of rollouts. Performance influences (\cref{eq:policy-inf-deriv}) converge around $m \in [25, 50]$ rollouts for ``Lift MH'' and ``Square MH'' (yielding similar quality trends), but continue to evolve until $m \in [50, 100]$ rollouts for ``Transport MH.''
        Curation performed on image-based diffusion policies.
        Results are averaged over 3 random seeds. 
        Errors bars represent the standard error.
    }
    \label{fig:robomimic-image-data-quality-rollout-appx}
\end{figure}

\subsection{Additional Data Quality Results in RoboMimic}

We provide full data quality results in RoboMimic. \cref{fig:robomimic-state-data-quality-appx} is identical to the top row of \cref{fig:robomimic-dp-results} in the main text, but also includes data quality trends for select-$k$ curation on ``Lift MH.'' \cref{fig:robomimic-image-data-quality-appx} shows data quality results for image-based diffusion policies. We do not retrain image-based policies on curated datasets (as in the bottom row of \cref{fig:robomimic-dp-results}) due to the substantial computational resources required.

\begin{figure}[H]
    \centering
    \includegraphics[width=\linewidth]{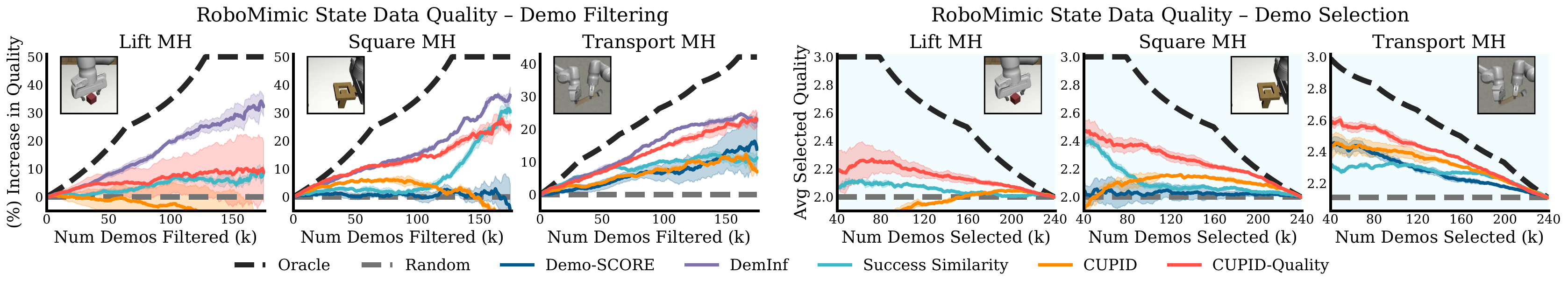}
    \vspace{-20pt}
    \caption{\small
        RoboMimic state data quality results. Curation performed on state-based diffusion policies. Results are averaged over 3 random seeds. Errors bars represent the standard error.
    }
    \label{fig:robomimic-state-data-quality-appx}
\end{figure}

\begin{figure}[H]
    \centering
    \includegraphics[width=\linewidth]{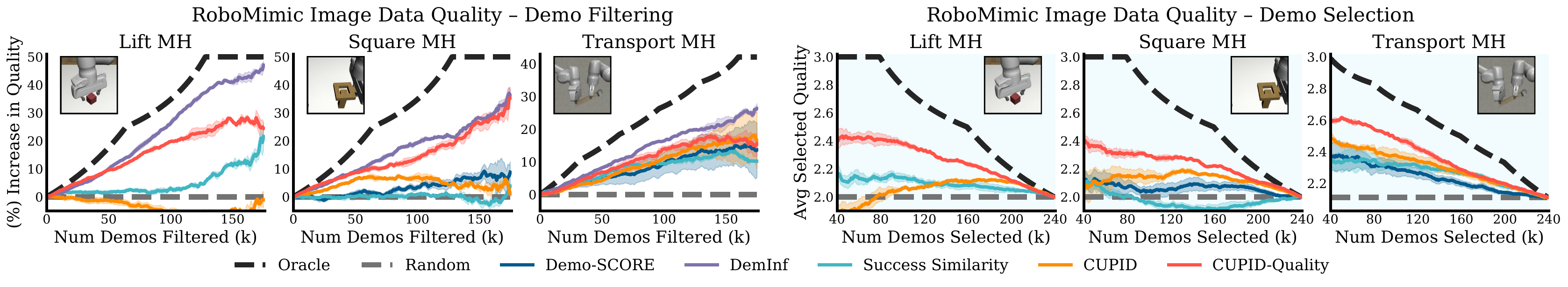}
    \vspace{-20pt}
    \caption{\small
        RoboMimic image data quality results. Curation performed on image-based diffusion policies. Results are averaged over 3 random seeds. Errors bars represent the standard error.
    }
    \label{fig:robomimic-image-data-quality-appx}
\end{figure}

\subsection{Data Filtering Curation Distributions in Franka Real-World}\label{appx:filter-distr}

\begin{figure}[H]
    \centering
    
    \begin{subfigure}[b]{\linewidth}
        \centering
        \includegraphics[width=0.95\linewidth]{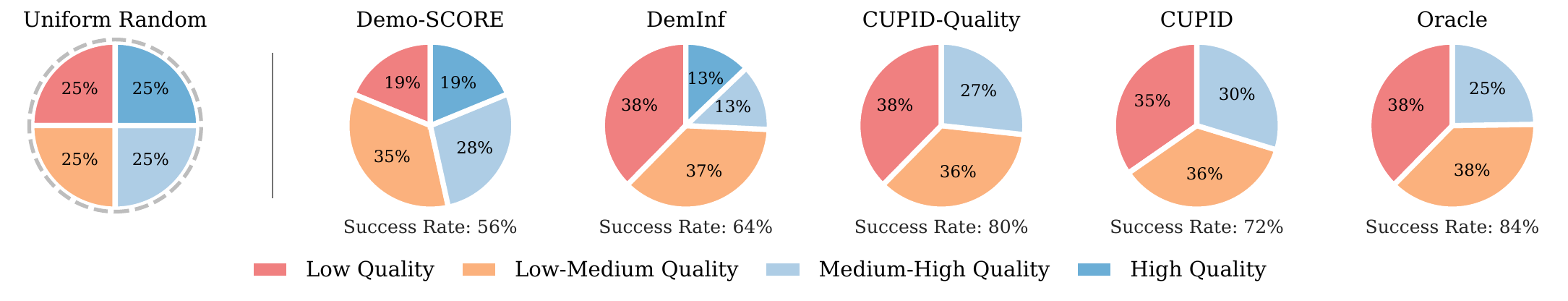}
        \vspace{-6pt}
        \caption{\footnotesize 
            \textbf{Figure-8:} Distribution of demonstrations \textit{filtered}. Filtering lower-quality demos is better.
        }
        \label{fig:franka-dp-distr-filter-curated-results-figure8}
    \end{subfigure}

    \vspace{18pt}

    \begin{subfigure}[b]{\linewidth}
        \centering
        \includegraphics[width=0.95\linewidth]{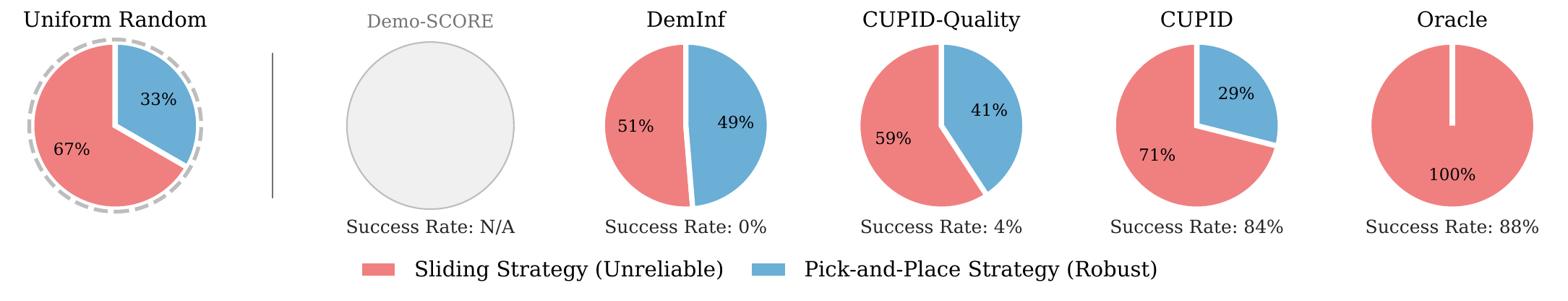}
        \vspace{-6pt}
        \caption{\footnotesize
            \textbf{TuckBox:} Distribution of demonstrations \textit{filtered}. Filtering sliding demos is better.
        }
        \label{fig:franka-dp-distr-filter-curated-results-tuckbox}
    \end{subfigure}
    
    \vspace{18pt}
    
    \begin{subfigure}[b]{\linewidth}
        \centering
        \includegraphics[width=0.95\linewidth]{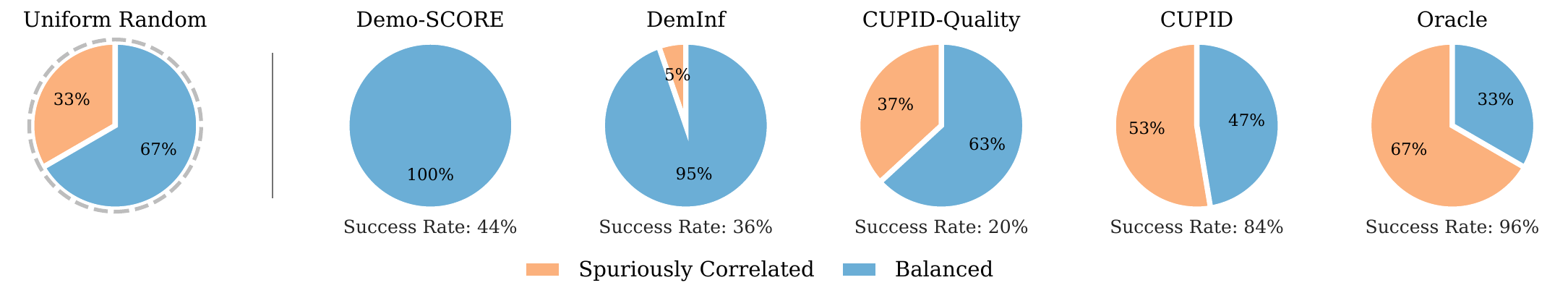}
        \vspace{-6pt}
        \caption{\footnotesize
            \textbf{Bookshelf:} Distribution of demonstrations \textit{filtered}. Filtering spurious correlations is better.
        }
        \label{fig:franka-dp-distr-filter-curated-results-bookshelf}
    \end{subfigure}
    
    \vspace{8pt}   
    
    \caption{\small
        \textbf{Franka diffusion policy -- distribution of demonstrations filtered ($S^\star$ in \cref{task:filter-k}).} See \cref{fig:franka-dp-distr-filter-dataset-results} for distributions of the corresponding curated datasets used for policy training.\\
    }
    \label{fig:franka-dp-distr-filter-curated-results}
\end{figure}

\subsection{Data Selection Curation Distributions in Franka Real-World}\label{appx:select-distr}

\begin{figure}[H]
    \centering
    
    \begin{subfigure}[b]{\linewidth}
        \centering
        \includegraphics[width=0.90\linewidth]{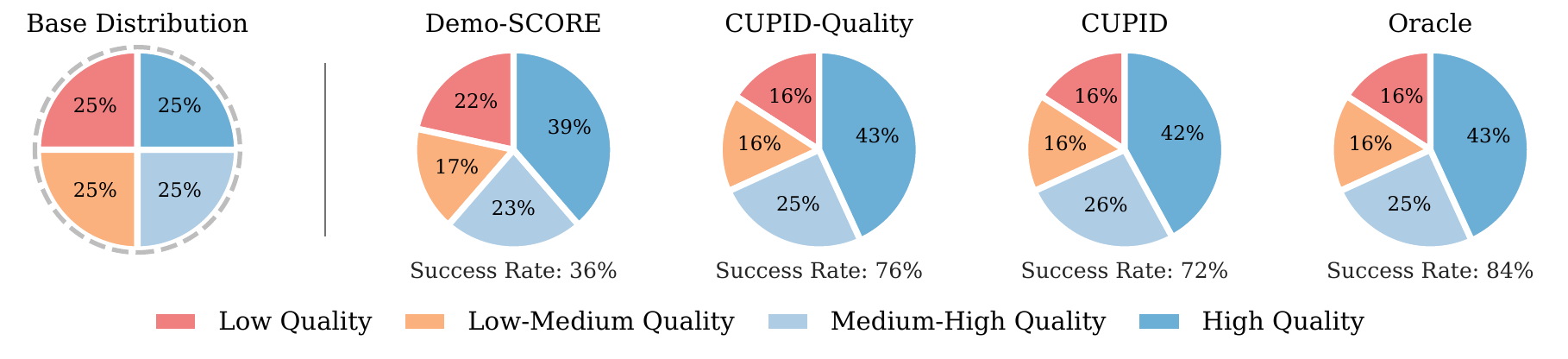}
        \vspace{-6pt}
        \caption{\footnotesize 
            \textbf{Figure-8:} Distribution of curated demonstrations after \textit{selecting} 33\%. Higher-quality demos are better.
        }
        \label{fig:franka-dp-distr-select-dataset-results-figure8}
    \end{subfigure}

    \vspace{18pt}

    \begin{subfigure}[b]{\linewidth}
        \centering
        \includegraphics[width=0.90\linewidth]{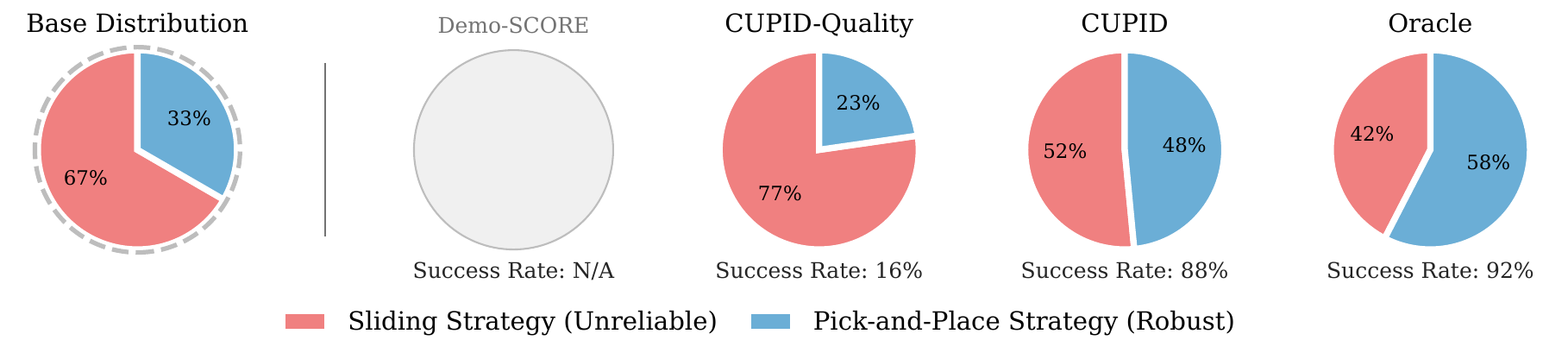}
        \vspace{-6pt}
        \caption{\footnotesize            
            \textbf{TuckBox:} Distribution of curated demonstrations after \textit{selecting} 33\%. Pick-and-place demos are better.
        }
        \label{fig:franka-dp-distr-select-dataset-results-tuckbox}
    \end{subfigure}
    
    \vspace{8pt}   
    
    \caption{\small
        \textbf{Franka diffusion policy curated dataset distributions for selection (\cref{task:select-k}).} \basemethod{} selects higher-quality demonstrations (Figure-8) and robust strategies (TuckBox), improving policy performance across tasks (see \cref{fig:franka-dp-results}). While curation heuristics employed by baselines may be effective in some cases (e.g., \qualitymethod{} in Figure-8), they can lead to suboptimal selection in others. \\
    }
    \label{fig:franka-dp-distr-select-dataset-results}
\end{figure}

\begin{figure}[H]
    \centering
    
    \begin{subfigure}[b]{\linewidth}
        \centering
        \includegraphics[width=0.90\linewidth]{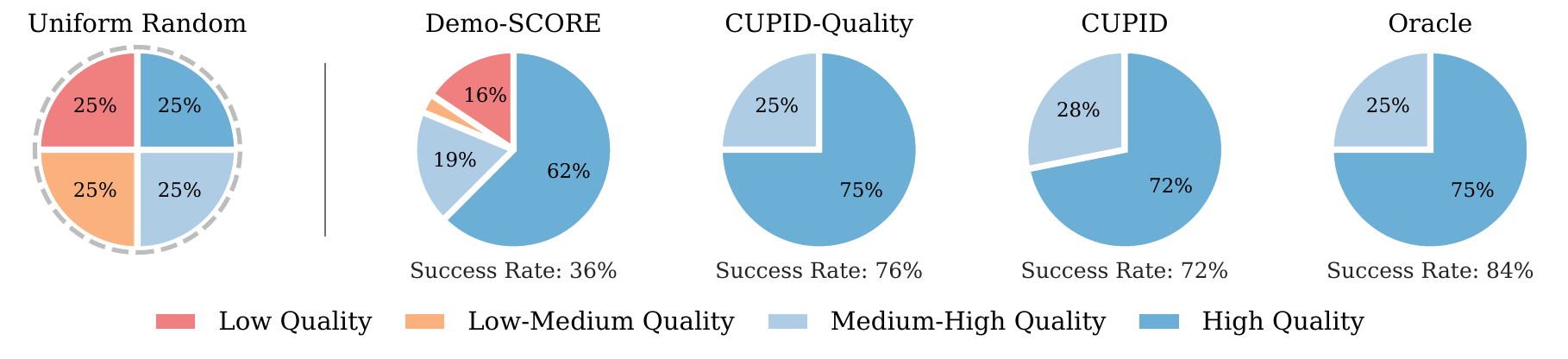}
        \vspace{-6pt}
        \caption{\footnotesize 
            \textbf{Figure-8:} Distribution of demonstrations \textit{selected}. Selecting higher-quality demos is better.
        }
        \label{fig:franka-dp-distr-select-curated-results-figure8}
    \end{subfigure}

    \vspace{18pt}

    \begin{subfigure}[b]{\linewidth}
        \centering
        \includegraphics[width=0.90\linewidth]{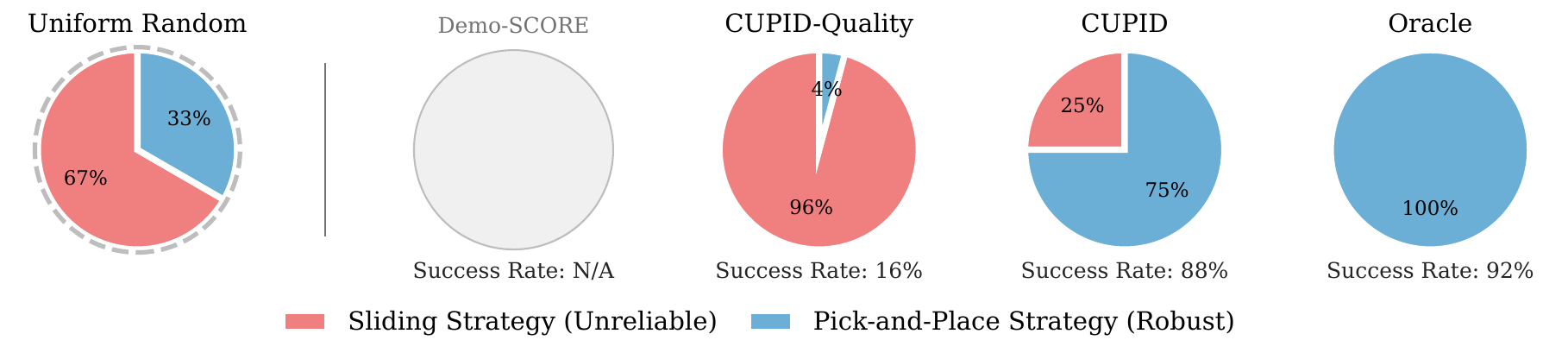}
        \vspace{-6pt}
        \caption{\footnotesize
            \textbf{TuckBox:} Distribution of demonstrations \textit{selected}. Selecting pick-and-place demos is better.
        }
        \label{fig:franka-dp-distr-select-curated-results-tuckbox}
    \end{subfigure}
    
    \vspace{8pt}   
    
    \caption{\small
        \textbf{Franka diffusion policy -- distribution of demonstrations selected ($S^\star$ in \cref{task:select-k}).} See \cref{fig:franka-dp-distr-select-dataset-results} for distributions of the corresponding curated datasets used for policy training.\\
    }
    \label{fig:franka-dp-distr-select-curated-results}
\end{figure}

\subsection{Additional Results for Franka $\pi_0$: Curated Dataset Transfer (\cref{sec:discussion-pi0-transfer})}\label{appx:results-pio-transfer}

\cref{fig:franka-pi0-results-appx} contains the full results of our $\pi_0$ ablation (\cref{fig:franka-pi0-transfer-results}), including the performance of $\pi_0$~\cite{black2410pi0} trained on datasets curated by \basemethod{} and \qualitymethod{} for both the ``Figure-8'' and ``TuckBox'' tasks. 

\begin{figure}[H]
    \centering
    \includegraphics[width=0.60\linewidth]{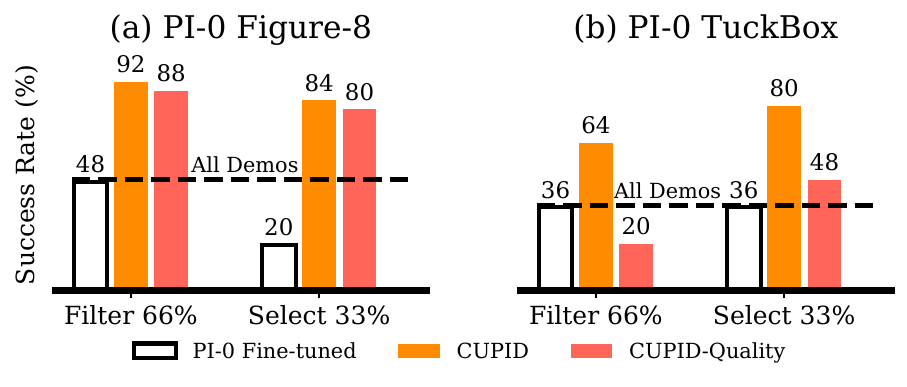}
    \vspace{-10pt}
    \caption{\small 
        Data curated for single-task diffusion policies improves $\pi_0$~\cite{black2410pi0} post-training performance. As in \cref{fig:franka-dp-results}, quality measures (\qualitymethod{}) may degrade performance when higher-quality demonstrations induce brittle strategies at test time (TuckBox), whereas curating based on performance (\basemethod{}) is robust across settings.
    }
    \label{fig:franka-pi0-results-appx}
\end{figure}

In this experiment, we investigate two questions: (1) Can datasets curated with one policy architecture result in increased performance when used to train another policy with a different architecture? (2) How influential is curation for policies that have been pre-trained on large-scale multi-task datasets?

\textit{Curation Transfer:} Towards the first question, \cref{fig:franka-pi0-results-appx} shows that datasets curated using diffusion policies significantly increase the performance of fine-tuned $\pi_0$ policies relative to fine-tuning on the base, uncurated datasets. We attribute these results to two causes: First, we find that both the diffusion policy and $\pi_0$ have sufficient capacity to accurately fit the training data distribution, and thus, they should learn a similar behavior distribution from the training data. This implies that the observed performance gains in \cref{fig:franka-pi0-results-appx} result from curation transfer between policies. Second, as the ``TuckBox'' experiment shows in \cref{fig:franka-dp-results}(b), our method is able to effectively identify behaviors in the demonstration data that are not robust. While on-policy evaluations (i.e., rollouts) are necessary to identify such brittle behaviors, these are purely properties of the training demonstration data. Therefore, filtering out poor behaviors will increase the performance of any policy.
Similarly, on the high-precision ``Figure-8'' task, filtering out more noisy, low-quality demonstrations is likely to improve performance for any policy. 

\textit{VLA Robustness:} Towards the second question, we find that even when the base policy is pre-trained on a large, diverse, multi-task dataset, curation is still essential to yield strong fine-tuned performance. As shown in \cref{fig:franka-pi0-results-appx}, $\pi_0$ policies trained on the base demonstration datasets are unable to reliably complete our tasks. In contrast, policies trained on curated datasets attain significantly higher success rates. As such, our results indicate that simply training VLM-based policies on more data and more tasks does not strictly result in pre-conditioned policies that use their generalist knowledge to ``ignore'' low-quality behaviors or brittle strategies in demonstration data---i.e., data curation still appears essential. 

\textit{Concluding Remarks:} Overall, these results indicate that using smaller, single-task policies to curate individual datasets, which may then benefit a larger, multi-task policy is a promising direction to alleviate the computational cost of applying our method to generalist policies. Still, we emphasize that datasets curated using our method are not completely \emph{model agnostic}, as the same demonstrations may influence different models in different ways. As such, while $\pi_0$ achieves a higher base performance than the diffusion policy, the $\pi_0$ policies trained on curated datasets perform similarly to or slightly worse than the diffusion policies (for which those datasets were curated).

\section{Derivations}\label{appx:derivations}

\subsection{Proof of \cref{prop:polinf}}\label{appx:proof}
\begin{proof}
    As presented in \cref{sec:background},  applying the basic derivation of the influence function\footnoteref{fn:track} in \cite{koh2017understanding} gives us that
    \begin{align*}
     \polinf(\xi) &:= \frac{dJ(\pi_\theta)}{d\epsilon}\bigg|_{\epsilon=0} \\
                  &= -\nabla_\theta J(\pi_\theta)^\top \nabla^2_\theta\calL_{\mathrm{bc}}(\theta ; \calD)^{-1} \nabla_\theta \ell_{\mathrm{traj}}(\xi; \pi_\theta).
     \end{align*}
     Next, note that the standard log-derivative trick underlying policy gradient methods \cite{sutton1999policy, williams1992simple} tells us that 
     \begin{align*}
         \nabla_\theta J(\pi_\theta) = \E_{\tau \sim p(\tau | \pi_\theta)} \big[R(\tau)\sum_{(s',a')\in \tau}\nabla_\theta\log \pi_\theta(a'|s')\big].
     \end{align*}
     Therefore, since $\calL_{\mathrm{bc}}$ and $\ell_{\mathrm{traj}}$ are deterministic functions of $\theta$, $\xi$, and $\calD$, it holds that 
     \begin{align*}
         \polinf(\xi) = \E_{\tau\sim p(\tau | \pi_\theta)} \big[ R(\tau) \sum_{(s',a')\in\tau} -\nabla_\theta\log \pi_\theta(a'|s')^\top H_{\mathrm{bc}}^{-1}\nabla_\theta\ell_{\mathrm{traj}}(\xi; \pi_\theta)\big]
     \end{align*}
     by linearity of expectation.
     Finally, by simply noting that $\ell_{\mathrm{traj}}(\xi; \pi_\theta) = \frac{1}{H}\sum_{(s,a)\in\xi}\ell(s, a;\theta)$ and applying the definition of $\actinf$, we have the result:
     \begin{align*}
     \polinf(\xi) = \E_{\tau\sim p(\tau|\pi_\theta)}\bigg[\frac{R(\tau)}{H} \sum_{(s',a')\in\tau}\sum_{(s,a)\in\xi}\actinf\big((s',a'),(s,a)\big)\bigg].
     \end{align*}
\end{proof}

\subsection{Derivation of Performance Influence for Variable Length Trajectories}\label{appx:proof-length}
In \cref{sec:formulation} and \cref{sec:method}, we assumed that all trajectories in the demonstration dataset $\calD$ were of an equal length $H$ for notational simplicity. Here, we show that without loss of generality, our analysis extends to the case where the length of demonstration trajectories vary. Suppose each demonstration $\xi^i \in \calD$ has length $H^i$, so that the base policy $\pi_\theta$ minimizes the average loss across all samples in the demonstration data, i.e., 
\begin{equation}\label{eq:bclossvary}
    \theta = \arg\min_{\theta'} \{\tilde{\mathcal{L}}_{\text{bc}}(\theta'; \mathcal{D}) := \frac{1}{(\sum_{i=1}^nH^i
    )
    }\sum_{\xi^i\in\calD}\sum_{(s, a) \in \xi^i} \ell(s, a; \pi_{\theta'})\}.
\end{equation}
Note that the objective in \cref{eq:bclossvary} is equivalent to an unweighted BC loss 
\begin{equation*}
    {\mathcal{L}'}_{\text{bc}}(\theta'; \mathcal{D}) := \sum_{\xi^i\in\calD}\sum_{(s, a) \in \xi^i} \ell(s, a; \pi_{\theta'}),
\end{equation*}
which decomposes into its unweighted trajectory losses $\ell_{\mathrm{traj}}'(\xi ; \pi_{\theta'}) := \sum_{(s,a) \in \xi}\ell(s, a; \pi_{\theta'})$, so that $\calL_{\mathrm{bc}}'(\theta', \calD) = \sum_{\xi^i \in \calD}\ell_{\mathrm{traj}}'(\xi^i ; \pi_{\theta'}).$ We can then derive an equivalent statement to \cref{prop:polinf} for the unweighted loss functions that applies when the demonstrations have variable length.

\begin{proposition}\label{prop:polinf-length}
Assume that $\theta(\calD) = \arg\min_{\theta'} \calL_{\mathrm{bc}}'(\theta'; \calD)$, that $\calL_{\mathrm{bc}}'$ is twice differentiable in $\theta$, and that $H_{\mathrm{bc}} \succ 0$ is positive definite (i.e., $\theta(\calD)$ is not a saddle point)\footnoteref{fn:track}. Then, it holds that 
\begin{equation}\label{eq:policy-inf-deriv-vary}
    \polinf(\xi) = \E_{\tau\sim p(\tau|\pi_\theta)}\bigg[R(\tau)\sum_{(s',a')\in\tau}\sum_{(s,a)\in\xi}\actinf\big((s',a'),(s,a)\big)\bigg].
\end{equation}
\end{proposition}
\begin{proof}
    As presented in \cref{sec:background},  applying the basic derivation of the influence function\footnoteref{fn:track} in \cite{koh2017understanding} gives us that
    \begin{align*}
     \polinf(\xi) &:= \frac{dJ(\pi_\theta)}{d\epsilon}\bigg|_{\epsilon=0} \\
                  &= -\nabla_\theta J(\pi_\theta)^\top \nabla^2_\theta\calL_{\mathrm{bc}}'(\theta ; \calD)^{-1} \nabla_\theta \ell_{\mathrm{traj}}'(\xi; \pi_\theta).
     \end{align*}
     Next, note that the standard log-derivative trick underlying policy gradient methods \cite{sutton1999policy, williams1992simple} tells us that 
     \begin{align*}
         \nabla_\theta J(\pi_\theta) = \E_{\tau \sim p(\tau | \pi_\theta)} \big[R(\tau)\sum_{(s',a')\in \tau}\nabla_\theta\log \pi_\theta(a'|s')\big].
     \end{align*}
     Therefore, since $\calL_{\mathrm{bc}}'$ and $\ell_{\mathrm{traj}}'$ are deterministic functions of $\theta$, $\xi$, and $\calD$, it holds that 
     \begin{align*}
         \polinf(\xi) = \E_{\tau\sim p(\tau | \pi_\theta)} \big[ R(\tau) \sum_{(s',a')\in\tau} -\nabla_\theta\log \pi_\theta(a'|s')^\top H_{\mathrm{bc}}^{-1}\nabla_\theta\ell_{\mathrm{traj}}'(\xi; \pi_\theta)\big]
     \end{align*}
     by linearity of expectation.
     Finally, by simply noting that $\ell_{\mathrm{traj}}'(\xi; \pi_\theta) = \sum_{(s,a)\in\xi}\ell(s, a;\theta)$ and applying the definition of $\actinf$, we have the result:
     \begin{align*}
     \polinf(\xi) = \E_{\tau\sim p(\tau|\pi_\theta)}\bigg[R(\tau)\sum_{(s',a')\in\tau}\sum_{(s,a)\in\xi}\actinf\big((s',a'),(s,a)\big)\bigg].
     \end{align*}
\end{proof}

\end{appendices}

\end{document}